\documentclass[11pt]{article}
\usepackage[T1]{fontenc}
\usepackage{lmodern}
\usepackage[left=2.cm,right=2.cm,top=2cm,bottom=2cm]{geometry}
\usepackage{authblk}

\author[$\dag\ddag$\footnote{Work done while visiting at Courant Institute.}]{Stefano Sarao Mannelli}
\author[$\ddag$]{Eric Vanden-Eijnden}
\author[$\dag$]{Lenka Zdeborov\'a}

\affil[$\dag$]{Universit\'e Paris-Saclay, CNRS, CEA, Institut de physique th\'eorique, 91191, Gif-sur-Yvette, France.}
\affil[$\ddag$]{Courant Institute, New York University, 251 Mercer Street, New York, New York 10012, USA}

\date{}

\usepackage{cite}
\usepackage[pdftex]{graphicx}

\usepackage[utf8]{inputenc} 
\usepackage[T1]{fontenc}    
\usepackage{hyperref}       
\hypersetup{colorlinks,linkcolor=blue, citecolor=red}
\usepackage{url}            
\usepackage{booktabs}       
\usepackage{amsfonts}       
\usepackage{nicefrac}       
\usepackage{microtype}      

\usepackage{xcolor}

\usepackage{amsmath}
\usepackage{amssymb}
\usepackage{bm}

\usepackage{float}
\usepackage{subfigure}

\usepackage{titlesec}      
\def\vb{\boldsymbol{v}}
\def\wb{\boldsymbol{w}}
\def\Wb{\boldsymbol{W}}
\def\xb{\boldsymbol{x}}

\def\zb{\boldsymbol{z}}

\def\Qm{A}

\def\Bm{B}

\def\Xm{X}

\def\<{\langle} \def\>{\rangle}

\def\RR{\mathbb{R}} \def\NN{\mathbb{N}} 
\def\EE{\mathbb{E}} \def\PP{\mathbb{P}} 
\DeclareMathOperator{\tr}{tr}

\DeclareMathOperator*{\argmin}{arg\,min}
\DeclareMathOperator*{\sign}{sign}

\usepackage{amsthm}
\newtheorem{theorem}{Theorem}[section]
\newtheorem{lemma}[theorem]{Lemma}
\newtheorem{proposition}[theorem]{Proposition}

\title{Optimization and Generalization of Shallow Neural Networks with Quadratic Activation Functions}

\begin{document}

\maketitle

\begin{abstract}
    We study the dynamics of optimization and the generalization properties of one-hidden layer neural networks with quadratic activation function in the over-parametrized regime where the layer width $m$ is larger than the input dimension $d$. 
    We consider a teacher-student scenario where the teacher has the same structure as the student with a hidden layer of smaller width $m^*\le m$. 
    We describe how the empirical loss landscape is affected by the number $n$ of data samples and the width $m^*$ of the teacher network. In particular we determine how the probability that there be no spurious minima on the empirical loss depends on $n$, $d$, and $m^*$, thereby establishing conditions under which the neural network can in principle recover the teacher. 
    We  also show  that under the same conditions gradient descent dynamics on the empirical loss converges and leads to small generalization error, i.e. it enables recovery in practice.
    Finally we characterize the time-convergence rate of gradient descent  in the limit of a large number of samples.
    These results are confirmed by numerical experiments.
\end{abstract}


%

\section{Introduction}

Neural networks are a key component of the  machine learning toolbox. Still the reasons behind their success remain  mysterious from a theoretical prospective. While sufficiently large neural networks can in principle represent a large class of functions,  we do not yet understand under what conditions their parameters can be adjusted in an algorithmically tractable way for that purpose.  For example, under worst case assumptions, some functions cannot be tractably learned with neural networks \cite{blum1989training,abbe2018provable}. We also know that there exist settings with adversarial initializations where neural networks fail in generalization to new samples, while the same setting from random initial conditions succeeds \cite{liu2019bad}. And yet, in many practical settings, neural networks are trained successfully even with simple local algorithm such as gradient descent (GD) or stochastic gradient descent (SGD). 

The problem of learning the parameters of a neural network is two-fold. First, we want that their training  on a set of data via minimization of a suitable loss function succeed in finding a set of parameters for which the value of the loss is close to its global minimum. Second, and more importantly, we want that such a set of parameters also generalizes well to unseen data. 
Theoretical guarantees have been obtained in many settings by a geometrical analysis of the loss showing that only global minima are present, see e.g. \cite{ge2016matrix,du2017gradient}. In particular it has been shown that network over-parametrization can be beneficial and lead to landscapes without spurious minima in which GD or SGD converge \cite{livni2014computational,lee2016gradient,soltanolkotabi2018theoretical,venturi2019spurious,sarao2019passed}. 
However, over-parametrized neural networks successfully optimized on a training set do not necessarily generalize  well -- for example neural networks can achieve zero errors in training without learning any rule \cite{zhang2016understanding}. It is therefore important to understand when zero training loss implies good generalization.  

It is know empirically that deep neural networks can learn functions that can be represented with a much smaller (sometimes even shallow) neural network \cite{ba2014deep,hinton2015distilling,frankle2018lottery}, but that learning the smaller network without first learning the larger one is computationally harder \cite{livni2014computational}. Our work provides a theoretical justification for this empirical observation by providing an explicit and rigorously analyzable case where this happens. 

\paragraph*{Main contributions:} In this work we investigate the issues of training and generalization in the context of a teacher-student set-up. We assume that both the teacher and the student are one-hidden layer neural network with quadratic activation function and quadratic loss. We focus on the over-parametrized or over-realizable case where the hidden layer of the teacher $m^*$ is smaller than that of the student $m$. 
We assume that the hidden layer of the student $m$ is larger than the dimensionality $d$, $m>d$, in that case:
\begin{itemize}
    \item We show that the value of the empirical loss is zero on all of its minimizers, but that the set of minimizers does not reduce to the singleton containing only the teacher network in general.
    \item We derive a critical value $\alpha_c = m^*+1$  of the number of samples $n$ per dimension $d$ above which the set of minimizers of the empirical loss has a positive probability to reduce to the singleton containing only the teacher network in the limit $n,d\to\infty$ with $n/d\ge \alpha_c$---i.e. we derive a sample complexity threshold above which the minimizer can have good generalization properties. The formula is proven for a teacher with a single hidden unit $m^*=1$ (a.k.a. phase retrieval). 
    \item We study gradient descent flow on the empirical loss starting from random initialization and show that it converges to a network  that can achieve perfect generalization above this sample complexity threshold $\alpha_c$. 
    \item We quantify the nonasymptotic convergence rate of gradient descent in the limit of large number of samples and show that the loss is bounded from above at all times by $C_1/(1+C_2t)$ for some constants $C_1,C_2>0$. We also evaluate the asymptotic convergence rate and identify two different regimes according to the input dimension and the number of hidden units, showing that in one case the loss converges as $O(t^{-2})$ as $t\to\infty$ while in the second case it converges exponentially.
    \item We show how the string method can be used to probe the empirical loss landscape and find minimum energy paths on this landscape connecting the initial weights of the student to those of the teacher, possibly going through flat portion or above energy barrier. This allows one to probe features of this landscape not accessible by standard GD.
\end{itemize}

In Sec.~\ref{sec:problem} we formally define the problem and derive some key properties that we use in the rest of the paper. In Sec.~\ref{sec:geom_considerations} we analyze the training and the generalization losses from the geometrical prospective, and derive the formula for the sample complexity threshold. In Sec.~\ref{sec:dynamics} we show that gradient descent flow can find good minima for datasets above this sample complexity threshold, and we characterize its convergence rate. In Sec.~\ref{sec:simulations} we present our results using the string method to probe the loss landscape. Finally in the appendix we give the proofs and some additional numerical results.




\paragraph*{Related works:} One-hidden layer neural networks with quadratic activation functions in the over-parametrized regime were considered in a range of previous works \cite{venturi2019spurious,du2018power,soltanolkotabi2018theoretical,haeffele2015global,nguyen2017loss}. Notably it was shown that all local minima are global when the number of hidden units $m$ is larger than the dimension $d$ and that gradient descent finds the global optimum \cite{du2018power,soltanolkotabi2018theoretical,haeffele2015global}, and also when the number of hidden units $m > \sqrt{2n}$ with $n$ being the number of samples \cite{du2018power,nguyen2017loss}. Most of  these results were established for arbitrary training data of input/output pairs, but consequently these works did not establish condition under which the minimizers reached by the gradient descent have good generalization properties. Indeed, it is intuitive that over-parametrization renders the optimization problem simpler, but it is rather non-intuitive that it does not destroy good generalization properties. 
In~\cite{du2018power}, under the assumption that the input data is  Gaussian i.i.d., a $O(1/\sqrt{n})$ generalization rate was established. However the generalization properties of neural networks with number of samples comparable to the dimensionality is mostly left open. 

Much tighter (Bayes-optimal) generalization properties of neural networks were established for data generated by the teacher-student model, for the generalized linear models in \cite{barbier2019optimal}, and for one hidden layer much smaller than the dimension in \cite{aubin2018committee}. However, these results were only shown to be achievable with approximate message passing algorithms and the performance of gradient-descent algorithm was not analyzed. 
Also studying over-parametrization with analogous tightness of generalization results is an open problem and has been achieved only for the one-pass stochastic gradient descent \cite{goldt2019dynamics}. 

A notable special case of our setting is when the teacher has only one hidden unit, in which case the teacher network is equivalent to the phase retrieval problem with random sensing matrix \cite{fienup1982phase}. For this case the performance of message passing algorithms is well understood and requires a number of samples linearly proportional to the dimension, $n > 1.13 d$ in the high-dimensional regime for perfect generalization \cite{barbier2019optimal}.  For randomly initialized gradient descent the best existing rigorous result for the phase retrieval requires $d {\rm poly} (\log{d})$ number of samples \cite{chen2019gradient}. The performance of the gradient-descent in the phase retrieval problem is studied in detail in a concurrent work \cite{anonymous20}, showing numerically that without overparametrization randomly initialized gradient descent needs at least $n \approx 7 d$ samples to find perfect generalization. In the present work we show that overparametrized neural networks are able to solve the phase retrieval problem with $ n > 2 d$ samples in the high-dimensional limit. This improves upon \cite{chen2019gradient} and falls close to the performance of the approximate message passing algorithm that is conjectured optimal among polynomial ones \cite{barbier2019optimal}. 
But most interesting is the comparison between our results for the phase retrieval obtained by overparametrized neural networks $\alpha_c=2$, and the results from \cite{anonymous20} who show that without overparametrized considerably larger $\alpha$ is needed for gradient descent to succeed to learn the same function. This comparison provides a theoretical justification for how overparametrization helps gradient descent to find good generalization properties with fewer samples. We stress that the same property would not apply to the message passing algorithms. We could speculate that more of the properties of overparametrization observed in deep learning are limited to the gradient-descent-based algorithms and would not hold for other algorithmic classes.  

Closely related to our work is Ref.~\cite{gamarnik2019stationary} in which the authors consider the same teacher-student problem as we do. The main difference is that they only consider teachers that have more hidden units than the input dimension, $m^*\ge d$, while we consider arbitrary $m^*$. As we show below the regime where  $m^* < d$ turns out to be interesting as it affects nontrivially the critical number of samples $n_c$ needed for recovery and leads to a more complex scenario in which $n_c$ depends also on $m^*$---in particular taking $m^*<d$ allows for recovery below the threshold $d(d+1)/2$, which is one of our main results. 


\section{Problem formulation}\label{sec:problem}

Consider a teacher-student scenario where a teacher network generates the dataset, and a student network aims at learning the function of the teacher. The teacher has weights $\wb^*_i\in\RR^d$, with $i=1,\ldots,m^*$. We will keep the teacher weights generic in most of the paper and  will specify them when needed, in particular for the simulations where we consider two specific teachers: one with $\{\wb^*_i\}_{i\le m^*}$ i.i.d. Gaussian with covariance  identity, and one with $\{\wb^*_i\}_{i\le m^*}$ orthonormal. 

The student's weights are $\wb_j\in \RR^d$, with $j=1,\dots,m$ and $m\ge d$. Given an input $\xb\in\RR^d$, teacher's and student's outputs are respectively
\begin{equation}
  \label{eq:3}
  f_*(\xb) = \frac1{m^*}\sum_{i=1}^{m^*} |\xb\cdot \wb_i^*|^2, \qquad\text{and}\qquad
  f(\xb) = \frac1m \sum_{j=1}^m |\xb\cdot \wb_j|^2,
\end{equation}
where we fixed the second layer of weights to $1/m^*$ and $1/m$, respectively. The teacher produces $n$ outputs $y_k=f_*(\xb_k) $ from random i.i.d. Gaussian samples $\xb_k\sim\nu = \mathcal{N}(0,{I_d})$, $k=1,\dots,n$. Given this dataset, we define the empirical loss
\begin{equation}
  \label{eq:loss_empirical}
  \begin{aligned}
    L_n(\wb_1,\ldots,\wb_m) 
    & = \frac14\EE_{\nu_n} \Big| \frac1{m^*} \sum_{i=1}^{m^*}|\xb\cdot \wb_i^*|^2- \frac1m
    \sum_{j=1}^m |\xb\cdot \wb_j|^2\Big| ^2
  \end{aligned}
\end{equation}
where $\EE_{\nu_n}$ denotes expectation
with respect to the empirical measure $\nu_n = n^{-1}\sum_{k=1}^n \delta_{\xb_k}$. As usual, the population loss is obtained by taking the expectation of~\eqref{eq:loss_empirical} with respect to $\nu$.

The student minimizes the empirical loss~\eqref{eq:loss_empirical} using gradient descent, $\dot \wb_i(t) =-m \partial_{\wb_i} L_n$. Explicitly
\begin{equation}
  \label{eq:w_evolution_empirical}
    \dot \wb_i(t) = \EE_{\nu_n} \big[\tr\left(\Xm(A^* -A(t)) \right) \Xm \wb_i(t)\big].
\end{equation}
where we introduced the following $d\times d$ matrices
\begin{equation}
    \label{eq:Aetc}
    A(t)= \frac1m \sum_{i=1}^m \wb_i(t) \wb_i^T(t), \quad A^* = \frac1{m^*} \sum_{i=1}^{m^*} \wb_i^* (\wb_i^*)^T, \quad \Xm = \xb\xb^T.
\end{equation}
We can now see that a closed equation for $A(t)$ can be derived from \eqref{eq:w_evolution_empirical}, and this new equation reduces the effective number of weights from $O(d n)$ to $O(d^2)$  without affecting neither the dynamics nor the other properties of the teacher and student since $f_*(\xb) = \tr(XA^*)$ and $f(\xb) = \tr(XA)$:
\begin{lemma}
  \label{lem:1}
  The GD flow~\eqref{eq:w_evolution_empirical} of the weights $\{\wb_i\}_{i\le m}$ on the empirical loss induces the following evolution equation for $A(t)$:
  \begin{equation}
    \label{eq:QGD}
    \dot{A} = -A \nabla E_n(A) - \nabla E_n(A) A = \EE_{\nu_n} [\tr\left(\Xm(A^*-A) \right) \left(A \Xm+\Xm A\right)],
  \end{equation}
  where $\nabla$ denotes gradient with respect to $A$ and $E_n(A)$ is twice the empirical loss~\eqref{eq:loss_empirical} rewritten in terms of $A$:
  \begin{equation}
    \label{eq:EPdef}
    E_n(A) = \frac12 \EE_{\nu_n} \left|\tr\left(\Xm(A-A^*) \right) \right|^2 .
  \end{equation}
\end{lemma}
%
It is also possible to write the equivalent of this lemma for the population loss:
\begin{lemma}
  \label{lem:1pop}
  The GD flow of the weights $\{\wb_i\}_{i\le m}$ on the population loss reads
  \begin{equation}
  \label{eq:w_evolution_pop}
    \dot \wb_i(t) = \tr(A^* -A(t)) \wb_i(t) + 2 (A^*-A(t)) \wb_i(t).
\end{equation}
and it induces the following evolution equation for $A(t)$:
  \begin{equation}
    \label{eq:QGDpop}
    \begin{aligned}
      \dot{A} & = -A \nabla E(A) -
      \nabla E(A) A
      = 2\left[(\tr(A^*-A) ) A + (A^*-A) A+ A
        (A^*-A)\right].
    \end{aligned}
  \end{equation}
  where  $E(A)$ is twice the population loss
  written in terms of $A$:
  \begin{equation}
    \label{eq:4}
    E(A) = \tr \left((A-A^*)^2\right) + \frac12 \left(\tr (A-A^*)\right)^2 .
  \end{equation}
\end{lemma}
Expression~\eqref{eq:4} for the population loss was already given in~\cite{gamarnik2019stationary}. Lemmas~\ref{lem:1} and \ref{lem:1pop} are proven in Appendices~\ref{si:L1} and \ref{si:L2}, respectively. In Appendix~\ref{si:prox} we also show that \eqref{eq:QGD} and \eqref{eq:QGDpop} are the continuous limit of proximal schemes on $E_n$ and $E$, respectively, relative to a specific Bergman divergence. 

\section{Geometrical Considerations and Sample Complexity Threshold}\label{sec:geom_considerations}

The empirical loss~$E_n(A)$ is quadratic, hence convex, with minimum zero. In addition  $A=A^*$ is a minimizer since $E_n(A^*)=0$. The main question we want to address next is when is this minimizer unique. 

Since the trace is a scalar product in the vector space of $d\times d$ matrices in which symmetric matrices form a $d(d+1)/2$ dimensional subspace, the empirical loss~$E_n(A)$ will be strictly convex in this subspace \textit{iff} we span it using $d(d+1)/2$ linearly independent $\Xm_k = \xb_k \xb_k^T$~\cite{gamarnik2019stationary}. Yet, if we restrict considerations to matrices $A$ that are also positive semidefinite, we need less data to guarantee that $A=A^*$ is the unique minimizer of $E_n(A)$, at least in some probabilistic sense: 
\begin{theorem}[Single unit teacher]\label{th:strict_convexity_loss}
    Consider a teacher with $m^*=1$ and a student with $m\ge d$ hidden units respectively, so that $A^*$ has rank 1 and $A$ has full rank. Given a data set $\{\xb_k\}_{k=1}^n$ with each $\xb_k \in \RR^d$ drawn independently from a standard Gaussian, denote by $\mathcal{M}_{n,d}$ the set of minimizer of the empirical loss constructed with $\{\xb_k\}_{k=1}^n$ over symmetric positive semidefinite matrices $A$, i.e.
    \begin{equation}
        \label{eq:minemp}
        \mathcal{M}_{n,d} = \left\{ A=A^T , \ \textit{positive semidefinite such that} \ E_n(A) = 0\right\}.
    \end{equation}
    Set $n = \lfloor\alpha d\rfloor$ for $\alpha\ge1$ and let $d\to\infty$. Then
    \begin{equation}
        \label{eq:lim1}
        \lim_{d\to\infty} \PP \left(\mathcal{M}_{\lfloor\alpha d\rfloor,d} \not = \{A^*\} \right) = 1 \qquad \text{if} \ \alpha \in [0,2]
    \end{equation}
    whereas
    \begin{equation}
        \label{eq:lim2}
        \lim_{d\to\infty} \PP \left(\mathcal{M}_{\lfloor\alpha d\rfloor,d} = \{A^*\}\right) >0 \qquad \text{if} \ \alpha \in (2,\infty).
    \end{equation}
\end{theorem}
In words, this theorem says that it exists a threshold value $\alpha_c=2$ such that  for any $n>n_c=\lfloor\alpha d\rfloor$ there is a finite probability that the empirical loss landscape trivializes and all spurious minima disappear in the limit as $d\to\infty$. For $n\le n_c$ however, this is not the case and spurious minima exist with probability 1 in the limit. Therefore, the chance to learn $A^*$ by minimizing the empirical loss from a random initial condition is zero if $\alpha \in [0,2)$ but it becomes positive if $\alpha>2$. The proof of Theorem~\ref{th:strict_convexity_loss} is presented in Appendix~\ref{si:proofs}. This proof shows that we can account for the constraint that $A$ be positive definite by making a connection with the classic problem of the number of extremal rays of proper convex polyhedral cones generated by a set of random vectors in general position. Interestingly, this proof also gives a criterion on the data set $\{\xb_k\}_{k=1}^n$ that guarantees that the only minimizer of the empirical loss be $A^*$: it suffices to check that the proper convex polyhedral cones constructed with the data vectors have a number of extremal rays that is less than~$n$.

\paragraph*{Heuristic extension for arbitrary $m^*$.} The result of Theorem~\ref{th:strict_convexity_loss} can also be understood via a heuristic algebraic argument that has the advantage that it applies to arbitrary $m^*$.  The idea, elaborated upon in Appendix~\ref{si:geometrical_consideration_heuristic}, is to count the number of constraints needed to ensure that the only minimum of the empirical loss is $A=A^*$, taking into account that (i) $A$ has full rank and $A^*$ has rank $m^*$ and (ii)  both $A$ and $A^*$ are positive semidefinite and symmetric,  so that  the number of negative eigenvalues of $A-A^*$ can at most be $m^*$. If we use a block representation of $A-A^*$  in which we diagonalize the block that contains the direction associated with the eigenvectors of $A-A^*$ with nonnegative eigenvalues, and simply count the number of nonzero entries in the resulting matrix (accounting for its symmetry), for 
$m^*<d$ we arrive at
\begin{equation}
    \label{eq:alphac}
    n_c = d(m^*+1) - \tfrac12m^*(m^*+1)
\end{equation}
while for $m^*\ge d$ we recover the result $n=d(d+1)/2$  already found in \cite{venturi2019spurious,gamarnik2019stationary}. Setting $n_c= \alpha_c d$ and sending $d \to \infty$, this gives the sample complexity threshold
\begin{equation}
    \label{eq:alpha_c}
    \alpha_c = (m^*+1) 
\end{equation}
which, for $m^*=1$, agrees with the result in Theorem~\ref{th:strict_convexity_loss}. The sample complexity threshold is confirmed in Fig.~\ref{fig:dyn_phases} via simulations using gradient descent (GD) on the empirical loss---we explain this figure in Sec.~\ref{sec:dynamics} after establishing that the GD dynamics converges.

\begin{figure}[!th]
	\centering
	\includegraphics[width=.5\linewidth]{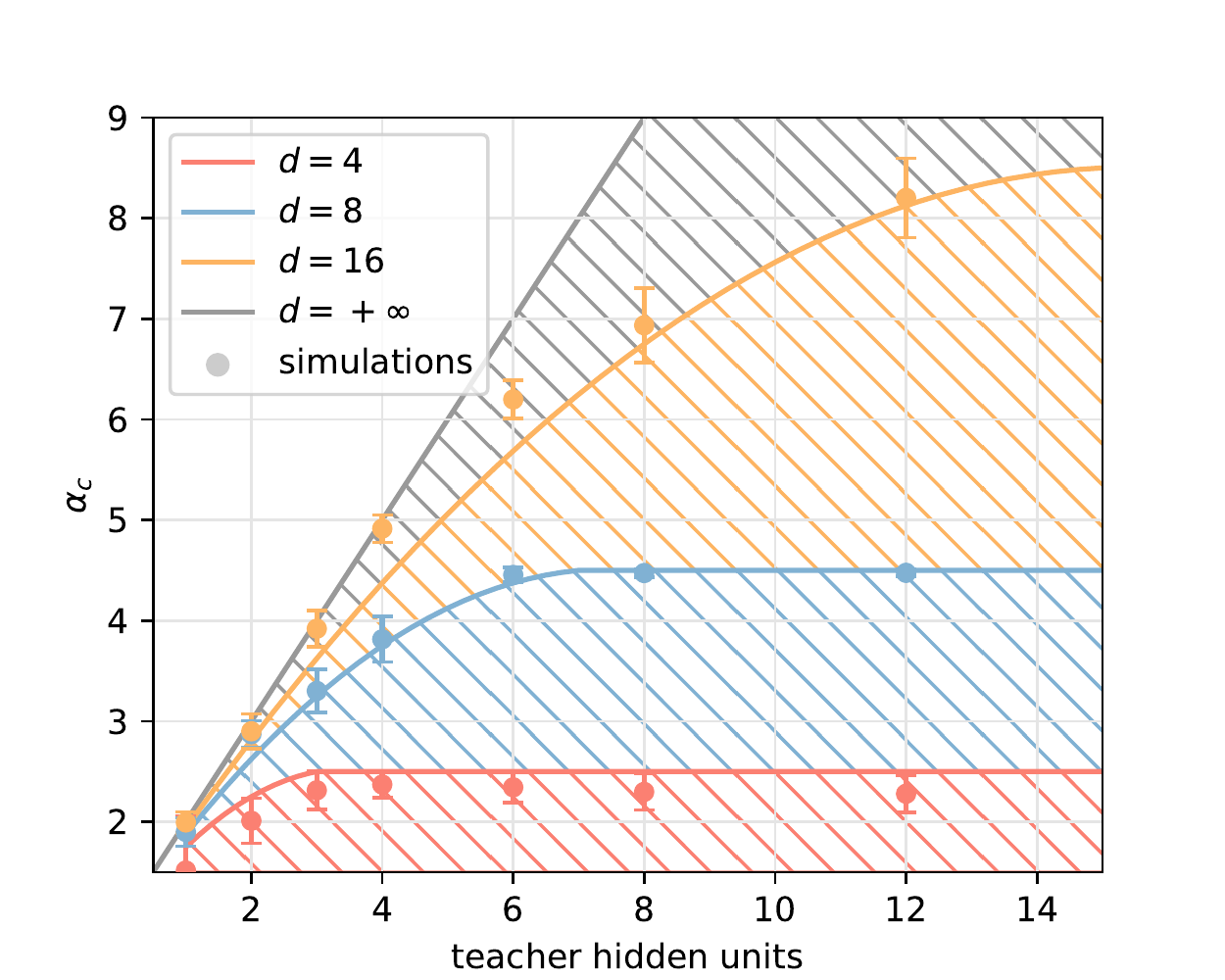}
	\caption{Dynamical phases of the student performance with a teacher having a number of hidden units given on the $x$-axis. The solid lines show the theoretical prediction in~\eqref{eq:alphac} for the sample complexity threshold and the points are obtained by extrapolation from  simulations with GD. In the simulations we consider a teacher with i.i.d. Gaussian weights and we report other cases in the Appendix.}
	\label{fig:dyn_phases}
\end{figure}

\section{Convergence of Gradient Descent on the Empirical Loss}\label{sec:dynamics}

Let us now analyze the performance of gradient descent over the empirical loss. As shown in Appendix~\ref{si:con}, we can  prove that:
\begin{theorem}
  \label{th:con}
  Let $\{\wb_i(t)\}_{i=1}^m$ be the solution to~\eqref{eq:w_evolution_empirical} for the initial data $\{\wb_i(0)\}_{i=1}^m$. Assume that $m\ge d$ and each $\wb_i(0)$ is drawn independently from a distribution that is absolutely continuous with respect to the Lebesgue measure on $\RR^d$.   Then 
  \begin{equation}
    \label{eq:convemp1}
      A = \frac1m \sum_{i=1}^m \wb_i(t) \wb_i^T(t) \to A_\infty =\frac1m \sum_{i=1}^m \wb^\infty_i (\wb_i^\infty)^T\quad \text{as\ \ $t\to\infty$}
  \end{equation}
  and $A_\infty$ is a global minimizer of the empirical loss, i.e.
  \begin{equation}
    \label{eq:convemp2}
      E_n(A_\infty) =  2 L_n (\wb_1^\infty,\ldots,\wb_n^\infty)= 0.
  \end{equation}
\end{theorem}
In a nutshell this theorem can be proven using the equivalence between the formulation using the weights with the GD flow in~\eqref{eq:w_evolution_empirical} over the loss $L_n$ in~\eqref{eq:loss_empirical} and that using $A$ with the evolution equation in~\eqref{eq:QGD} and the associated loss $E_n$ in~\eqref{eq:EPdef}. We can invoke the Stable Manifold Theorem~\cite{smale1963stable} to assert that the solution~\eqref{eq:w_evolution_empirical} must converge to a local minimum of $L_n$; as soon as $m\ge d$ and $A(0)$ has full rank, this minimum must be a minimum of $E_n$, which means that it must be the global since $E_n$ is convex. Note also that Theorem~\ref{th:con} can be generalized to time-discretized version of the GD flow using the results in Ref.~\cite{lee2016gradient}

Combined with Theorem~\ref{th:strict_convexity_loss}, Theorem~\ref{th:con} indicates that, when $m^*=1$ and $d$ is large, the probability that $A_\infty \not = A^*$ is high when $n/d \ge 2$, whereas the probability that $A_\infty = A^*$ becomes positive for $n/d>2$.  If we generalize this analysis to the case $m^*>1$ and $d$ large, we expect that GD will recover the teacher only if $n\ge n_c$ with $n_c$ given by~\eqref{eq:alphac}. 

These results are confirmed by numerical simulations in Fig.~\ref{fig:dyn_phases} where we plot $\alpha_c=n_c/d$ as a function of the number of teacher hidden units $m^*$ for different values of $d$. The four colors represent different input dimensions $d=4,8,16,\infty$. We use circles to represent the numerical extrapolation of $\alpha_c$ obtained by several runs of GD flow on different instances of the problem, using the procedure described in Appendix~\ref{si:numerical}. Consistent with Theorem~\ref{th:con}, the extrapolation confirms that GD flow is able match the sample complexity threshold predicted by the theory.



\section{Convergence Rate of Gradient Descent on the Population Loss}\label{sec:dynamics2}
\begin{figure}
	\centering
	\includegraphics[width=1.\linewidth]{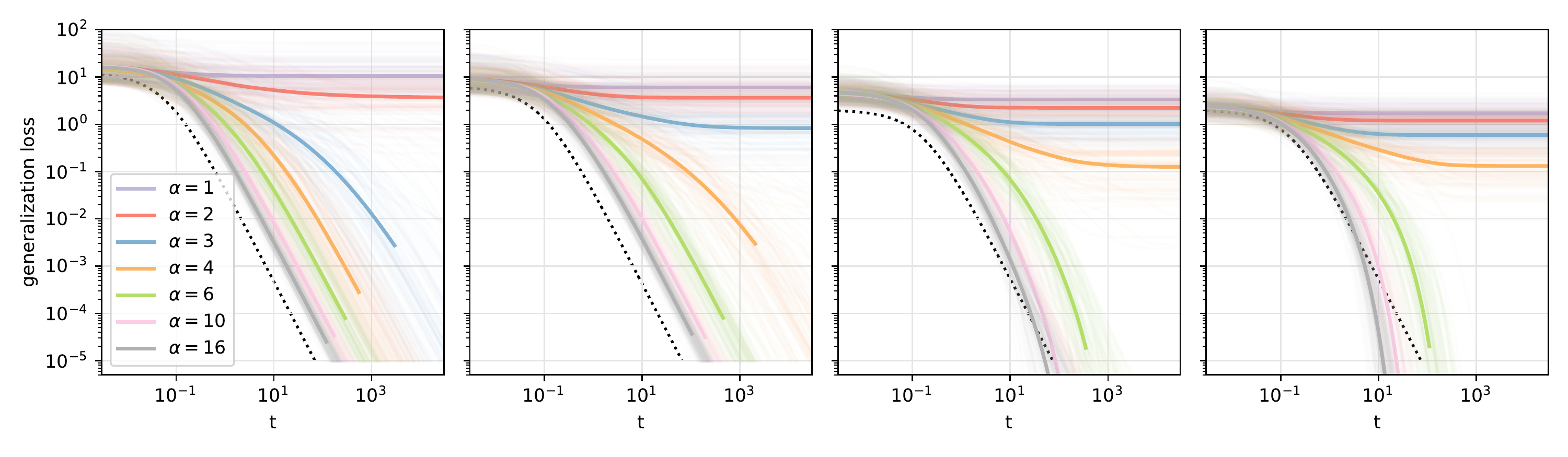}
	\caption{Convergence rates increasing the number of hidden units in the teacher $m^*$. The figures show log-average of $100$ simulations with $d=8$ and from left to right $m^*=2,4,8,16$, respectively. The individual simulations are shown in transparency. The dotted line is the quadratic decay and serves as reference. The figure shows that, if $\alpha>\alpha_c$ and $m^*>d-1$ the convergence rate becomes faster than quadratic, and in fact exponential as derived in Sec.~\ref{sec:dynamics}.}
	\label{fig:generalization_loss_m}
\end{figure}

Theorems~\ref{th:con} leaves open is the convergence rate of $A(t)$ towards $A_\infty$. This question is hard to answer for GD on the empirical loss, but it can be addressed for GD on the population loss. 

\begin{theorem}\label{th:convrate} Let $\{\wb_i(t)\}_{i=1}^m$ be the solution to~\eqref{eq:w_evolution_pop} for the initial data $\{\wb_i(0)\}_{i=1}^m$. Assume that $m\ge d$ and each $\wb_i(0)$ is drawn independently from a distribution that is absolutely continuous with respect to the Lebesgue measure on $\RR^d$.  Then 
\begin{equation}
    \label{eq:convpop1}
      A(t) = \frac1m \sum_{i=1}^m \wb_i(t) \wb_i^T(t) \to A^*\quad \text{as\ \ $t\to\infty$}
  \end{equation}
and we have the following nonasymptotic bound on the convergence rate of the population loss~\eqref{eq:4}:
\begin{equation}
\label{eq:Eboundnonas}
  \exists C >0 \quad : \quad \forall t \ge 0 \qquad E(A(t)) \le \frac{E(A(0))}{1+2 C
    E(A(0)) t  }
\end{equation}
In addition  $E(A(t))$ decays faster than $1/t$ as $t\to\infty$, i.e. $E(A(t))=o(1/t)$ and
\begin{equation}
\label{eq:Eboundas}
    \lim_{t\to\infty} t E(A(t)) = 0.
\end{equation}
\end{theorem}

This theorem is proven in Appendix~\ref{si:conrate}. The proof uses the convexity of $E(A)$ and deals with the added complexity of the factors $A$ multiplying $\nabla E$ in~\eqref{eq:QGDpop}. The argument also uses a stochastic representation formula for~$A^{-1}(t)$ given in Lemma~\ref{lem:trace} which is interesting in its own right.  We stress that~\eqref{eq:Eboundnonas} holds even when $m^*<d$, i.e. when $A^*$ is rank deficient, which is the difficult case for analysis since the factors $A$ multiplying $\nabla E$ in~\eqref{eq:QGDpop} converge to $A^*$ and hence become only positive semidefinite (as opposed to positive definite) as $t\to\infty$.

Theorem~\ref{th:convrate} holds for arbitrary initial conditions $A(0)$ with full rank. If the initial weights $\wb_i(0)$ are drawn independent from a standard Gaussian distribution in $\RR^d$, we know that $A(0) = m^{-1} \sum_{i=1}^m \wb_i(0) \wb^T_i(0) \to \text{Id}$ almost surely as $m\to\infty$ by the Law of Large Numbers. Therefore it makes sense to consider the GD flow~\eqref{eq:QGDpop} on the population loss when $A(0)=\text{Id}$. In that case, we have:

\begin{theorem}\label{th:QGDpop_dynamical_evolution}
    Let $A(t)$ be the solution to ~\eqref{eq:QGDpop} for the initial condition $A(0) = \text{Id}$. Denote by $U^*$ an orthogonal matrix whose columns are the eigenvectors of $A^*$, so that $A^* = U^* \Lambda^* (U^*)^T$ with $\Lambda^* = \text{diag}(\lambda^*_1, \ldots, \lambda^*_d)$. Let $\Lambda(t) = (U^*)^T A(t) U^*$ so that $\Lambda(0) = \text{Id}$. Then $\Lambda(t)$ remains diagonal during the dynamics and the evolution of its entries is given by
    \begin{equation}
        \label{eq:Lambdacomp}
        \begin{aligned}
            \dot\lambda_i & = 2 \sum_{j=1}^d (\lambda_j^*-\lambda_j)  \lambda_i + 4(\lambda_i^*-\lambda_i)  \lambda_i,\quad \lambda_i(0)=1, \qquad i=1,\ldots, d.
        \end{aligned}
    \end{equation}
    In addition the population loss is given by
\begin{equation}
\label{eq:lpop_eigen}
    E[A(t)]=\sum_{j=1}^{d}(\lambda_j(t)-\lambda_j^*)^2 + \frac12\Big(\sum_{j=1}^{d}\lambda_j(t)-\lambda_j^*\Big)^2.
\end{equation}
\end{theorem}
This theorem is proven in Appendix~\ref{si:popdiag}.  The equations in~\eqref{eq:Lambdacomp} can easily be solved numerically.  A formal asymptotic analysis of their solution when  $d$ is large is also possible, as shown next. This analysis  characterizes the asymptotic convergence rate of the eigenvalue to the target, which  can be used to obtain  an asymptotic convergence rate of the loss that is more precise than~\eqref{eq:Eboundas}: Specifically, it shows that $E(A(t))$ eventually decays as $1/t^2$ when $m^*<d$ and exponentially fast in $t$ when $m^*\ge d$.

\subsection{Formal asymptotic analysis of~\eqref{eq:Lambdacomp}}

\paragraph*{Case $m^*\ll d$.} Then $d-m^*$ eigenvalues of $A^*$ are zero, and without loss of generality we can order $\{\lambda_i\}_{i\le d} $ so that the zero eigenvalues of $A^*$ are last. Denoting $\epsilon(t) = \frac1{d-m^*}\sum_{i=m^*+1}^d \lambda_i(t)$, for $m^*<d$ \eqref{eq:Lambdacomp} then reads
\begin{align}
	\label{eq:eigenvalue_evolution_asymptotic_leada}
	\dot\lambda_i  & = 2\Big(\sum_{j=1}^{m^*}(\lambda^*_j-\lambda_j) - (d-m^*) \epsilon\Big)\lambda_i + 4(\lambda^*_i-\lambda_i)\lambda_i,\qquad i=1,\ldots, m^*
	\\
	\label{eq:eigenvalue_evolution_asymptotic_meanaa}
	\Dot{\epsilon} &  = 2\Big(\sum_{j=1}^{m^*}(\lambda^*_j-\lambda_j) - (d-m^*) \epsilon\Big)\epsilon - \frac{4}{d-m^*}\sum_{j=m^*+1}^d\lambda_j^2, \qquad \epsilon(0) = 1.
\end{align}
We will call the first $m^*$ eigenvalues $\lambda_i$ \textit{informative eigenvalues} and the remaining $d-m^*$ (captured by $\epsilon(t)$) \textit{non-informative eigenvalues}. We make two observations. Since $\lambda_i(0)=\epsilon(0)=1$, initially the leading order term in the equation for the uninformative eigenvalues $\epsilon(t)$ is 
\begin{equation}
    \label{eq:leadingeps}
    \dot \epsilon \approx -2d \epsilon^2\qquad \Rightarrow \qquad \epsilon(t) \approx \frac1{1+2dt} \qquad t\ll 1/d
\end{equation}
Substituting this solution into~\eqref{eq:eigenvalue_evolution_asymptotic_leada} we deduce
\begin{equation}
	\label{eq:eigenvalue_evolution_asymptotic_lead2}
	\frac{d}{dt}\log\lambda_i \approx -2d \epsilon(t) \approx -\frac{2d}{1+2dt} \qquad \Rightarrow \qquad 
	\lambda_i(t) \approx \frac1{1+2dt}
\end{equation}
\eqref{eq:leadingeps} and \eqref{eq:eigenvalue_evolution_asymptotic_lead2} imply an initial decreases in time of both non-informative and the informative eigenvalues. However, when $2d/(1+2dt)$ becomes of order one or smaller, the other terms in equation~\eqref{eq:eigenvalue_evolution_asymptotic_leada} take over and allow the informative eigenvalues to bounce back up. This happens at at time  $t_0 =O(1) $ in $d$. Afterwards the informative eigenvalues emerge from the non-informative ones with an exponential growth,  $\lambda_j(t)\sim \frac1{2d}e^{(2m^*+4)t}$. As a result, these informative eigenvalues eventually match the eigenvalues of the teacher at a typical time of order $t_J \sim \frac1{2m^*+4}\log(2d)$.  This analysis also implies a quadratic decay in time of the loss at long times
\begin{equation}
    \label{eq:loss_alpha}
    E(A(t)) \sim 1/(16t^2) \qquad \text{as} \ \ t\to\infty.
\end{equation}
In Sec.~\ref{si:numerical} we give additional details comparing the asymptotic analysis to the real dynamics when $m^*\le d$ but not necessarily much smaller. This analysis can e.g. be done quite explicitly when the unit in the teacher are orthonormal. It indicates that  $\epsilon(t) \approx 1/{[1+2(2+ d-m^*)t]}$ at all times, and as a result shows that
\begin{equation}\label{eq:loss_alpha2}
	E[A(t)] \approx \frac14\left(\frac{d-m^*}{1+2(2+d-m^*)t}\right)^2
\end{equation}
at all times.

\paragraph*{Case with $m^*\ge d\gg 1$.} Then \eqref{eq:Lambdacomp} can be written as
\begin{equation}
	\frac{d}{dt} \log\lambda_i  = 4\lambda_i^* + 2\sum_{j=1}^{d}\lambda^*_j -4\lambda_i-2\sum_{j=1}^{d}\lambda_j,\qquad i=1,\ldots, d
\end{equation}
which gives an exponential convergence to the target $A^*$, and consequently an exponential convergence in the population loss. For example, let us specialize to the case of a teacher with orthonormal hidden vectors, $\lambda^*_j = 1$ for $j=1,\dots,\min(m^*,d)$. The eigenvalues will converge to their target value as $|\lambda_j(t)-\lambda_j^*|\sim \frac1{2d}e^{-(2d+4)t}$. Consequently the loss~\eqref{eq:lpop_eigen} will converge to zero exponentially in this case 
\begin{equation}
\label{eq:loss_alpha_infty}
    E[A(t)] \sim \frac1{2d}e^{-2(2d+4)t} \qquad \text{as} \ \ t\to\infty.
\end{equation}

The results above are confirmed in the numerics. The cases when $m^*<d$ and  $m^*\ge d$ are shown by the first two and last two panels in Fig.~\ref{fig:generalization_loss_m}, respectively. When $m^*<d$ the decay of the empirical loss is quadratic, consistent with~\eqref{eq:loss_alpha}. In contrast, when $m^*\ge d$, the absence of non-informative eigenvalues removes the dominating terms in the loss~\eqref{eq:lpop_eigen}. Therefore the loss is dominated by the informative eigenvalues and decays exponentially, consistent with~\eqref{eq:loss_alpha_infty}. This can be clearly observed in Fig.~\ref{fig:generalization_loss_m}, where the four panels show the population loss using teachers with $m^*=2,4,8,16$ and $d=8$. The black dotted shows the quadratic asymptotic decay predicted in~\eqref{eq:loss_alpha}. The last two panels of the sequence show the exponential decay as predicted predicted in~\eqref{eq:loss_alpha_infty}

\begin{figure}
	\centering
	\includegraphics[width=.9\linewidth]{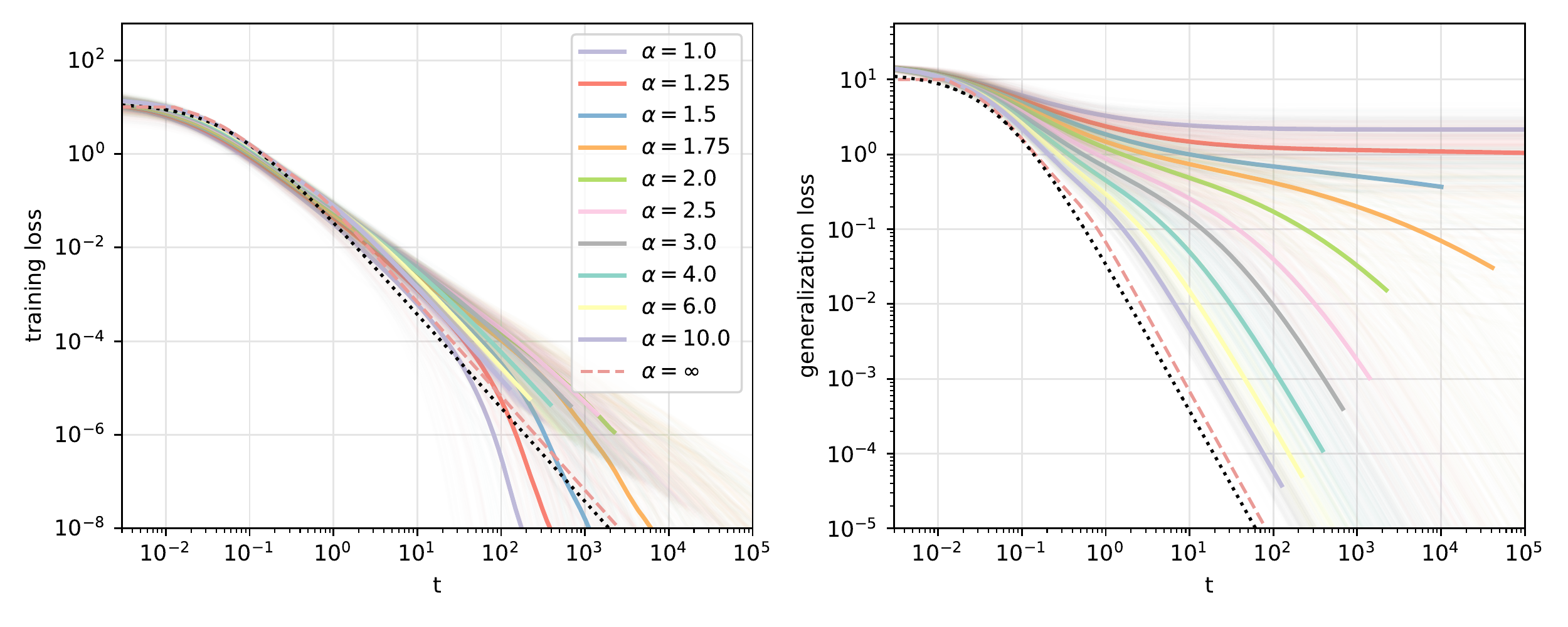}
	\caption{Training loss (left figure) and population loss (right figure) for $d=8$ and $m^*=1$. The plots show the average in log-scale of 100 simulation for each value of $\alpha$ and the individual realizations are shown in transparency. The results are compared with the descent in the population loss Eq.~\eqref{eq:QGDpop} (dashed pink line) and its approximation Eq.~\eqref{eq:loss_alpha2} (black dotted line). }
	\label{fig:loss_g_t_comparison}
\end{figure}

Fig.~\ref{fig:loss_g_t_comparison} shows the training and the population loss observed in the simulation using input dimension $d=8$ and a teacher with $m^*=1$ hidden unit. In this case our analysis suggests that the typical realization will converge to zero generalization error if $\alpha>\alpha_c=1.875$. This can be observed on the right panel of the Fig.~~\ref{fig:loss_g_t_comparison}. 
We used a dashed line to represent the gradient in the population loss~\eqref{eq:QGDpop} and used a dotted line to represent the approximated result~\eqref{eq:loss_alpha2}, observing the two being almost indistinguishable in the figure.

\section{Probing the Loss Landscape with the String Method}
\label{sec:simulations}

\begin{figure}
	\centering
	\includegraphics[width=1.\linewidth]{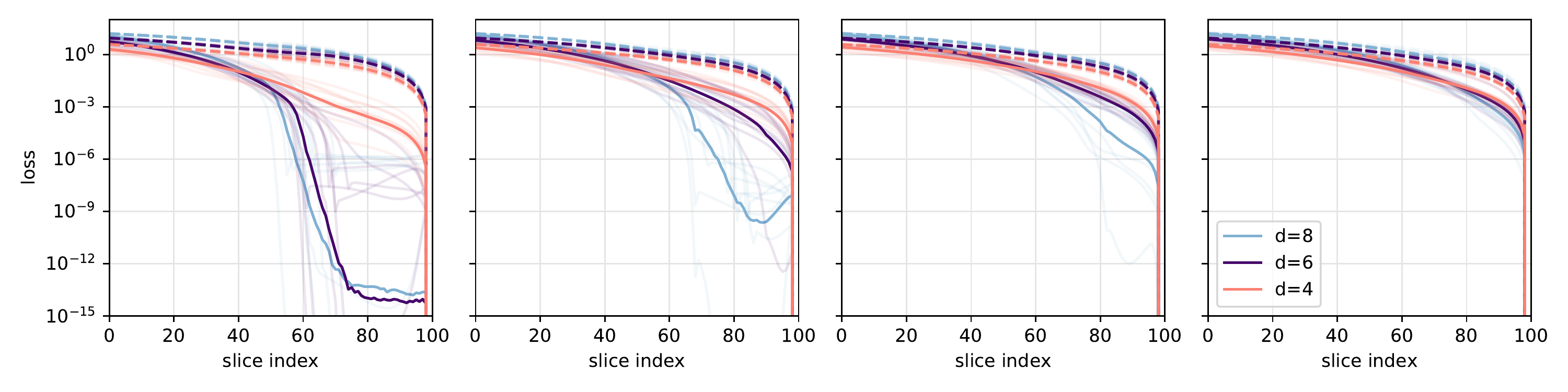}
	\caption{Results from the application of the string method. Training loss (solid line) and population loss (dashed line) evaluated across a string discretized with 100 images. Moving from left to right panels, the number of samples in the dataset increases, respectively $n=8,12,16,20$, while the teacher always has $m^*=1$ hidden units. The critical size to obtain a smooth landscape in average is $n=2d-1$, which is confirmed by the string reaching zero empirical loss at a finite value of the population loss, or not. Each string is mediated in log-scale over 10 realizations.}
	\label{fig:string}
\end{figure}

Finally, let us show that we can use the string method \cite{weinan2002string,weinan2007simplified,freeman2016topology} to probe the geometry of the training loss landscape and  confirm numerically Theorem~\ref{th:strict_convexity_loss}. The string method consists in connecting the student and the teacher with a curve (or string) in matrix space, and evolve this curve by GD while controlling its parametrization. In practice, this can be done efficiently  by discretizing the string into equidistant images or points (with the Frobenius norm as metric), and iterating upon (i) evolving these images by the descent dynamics, and (ii) reparamterizing the string to make the images equidistant again. At convergence the string will identify a minimum energy path between $A(0)$ and $\Qm^*$ which will possibly have a flat portion at zero empirical loss if this loss can be minimized by GD before reaching $A^*$. That is, along the string, the student  $A$ reaches the first minimum $A_\infty$ by GD, and, if $A_\infty\not = A^*$, then move along the set of minimizers of the empirical loss until it reaches $A^*$. The advantage of the method is that by replacing the physical time along the trajectory by the arclenght along it, it permits to go to infinite times (when $A=A^\infty$) and beyond  (when $A^\infty\not= A^*$), thereby probing features of the loss landscape not accessible by standard GD. (Of course it requires one to know the target $A^*$ in advance, i.e. the string method cannot be used instead of GD to identify this target in situations where it is unknown.) 

In Fig.~\ref{fig:string} we compare the strings obtained for input dimension 4 (red), 6 (purple), end 8 (blue). The strings are parametrized by 100 points represented on the horizontal axes. Moving from the leftmost to the rightmost panels in Fig.~\ref{fig:string} the number of samples in the dataset increases, namely $n=8,12,16,20$. Gradually all the $d$ represented will reach the critical size $2d-1$ and will have a landscape with a single minimum, the informative one. 
Observe that for relatively small sample sizes, there is low correspondence between the topology of the training loss landscape and the population loss one. As the size increases we notice that correlation increases until the two are just slightly apart. 






\section*{Acknowledgements}
	
	We thank Joan Bruna and Ilias Zadik for precious discussions. 
	SSM acknowledges the Courant Institute for the hospitality during his visit.
	We acknowledge funding from the ERC under the European
	Union’s Horizon 2020 Research and Innovation Programme Grant
	Agreement 714608-SMiLe.


\newpage
\appendix
\numberwithin{equation}{section}
\numberwithin{figure}{section}

\section{Proofs and Technical Lemmas}\label{si:proofs}


\subsection{Proof of Lemma~\ref{lem:1}}
\label{si:L1}

Inserting  $A(t)$ as defined in~\eqref{eq:Aetc} into~\eqref{eq:w_evolution_empirical} we arrive at 
\begin{equation}
    \label{eq:Astap1}
    \begin{aligned}
    \dot A(t) & = \frac1m\sum_{i=1}^m \EE_{\nu_n} \big[\tr\left(\Xm(A^* -A(t)) \right) \left(\Xm \wb_i(t) \wb^T_i(t) + \wb_i(t) \wb^T_i(t) \Xm\right) \big]\\
    & = \EE_{\nu_n} \big[\tr\left(\Xm(A^* -A(t)) \right) \left(\Xm A(t) + A(t) \Xm\right) \big]
    \end{aligned} 
\end{equation}
where we used $\Xm^T = \Xm$. This proves that $\dot A(t)$ is equal the the rightmost equation in~\eqref{eq:QGD}. To prove the first equality, simply note that, from~\eqref{eq:EPdef}, 
\begin{equation}
    \label{eq:Ediff}
    \nabla E_n(A) = -\EE_{\nu_n} \big[\tr\left(\Xm(A^* -A(t)) \right) \Xm  \big]
\end{equation}
which shows that~\eqref{eq:Astap1} can be written as $\dot A= - A \nabla E_n - \nabla E_n A$. \hfill$\square$

\subsection{Proof of Lemma~\ref{lem:1pop}}
\label{si:L2}

Equations~\eqref{eq:QGDpop} and~\eqref{eq:4} can be derived from~\eqref{eq:QGD} and~\eqref{eq:EPdef} by taking their expectation over $\nu$, owing to  the fact that the data is Gaussian and using Wick's theorem which asserts that
\begin{equation}
    \label{eq:wick}
    \EE_{\nu} [\Xm_{i,j} \Xm_{k,l}] = \delta_{i,j}\delta_{k,l} +  \delta_{i,k} \delta_{j,l} + \delta_{i,l} \delta_{j,k}
\end{equation}
This gives the result since $A^*$ and $A(t)$ are symmetric matrices. Note that this derivation can be generalized to non-Gaussian data, see Ref.~\cite{gamarnik2019stationary} for details. \hfill$\square$

\subsection{Proximal scheme}
\label{si:prox}

We note that \eqref{eq:QGD} (and similarly~\eqref{eq:QGDpop} if we use the population loss in~\eqref{eq:4} instead of the empirical loss in~\eqref{eq:EPdef}) can be viewed as the time continuous limit of a simple proximal scheme involving the Cholesky decomposition of $A$ and the standard Forbenius norm as Bregman distance. We state this result as:

\begin{proposition}
  \label{lem:proximal_scheme}
  Given $\Bm_0\in \RR^{d\times d}$ define the sequence of matrices
  $\{\Bm_p\}_{p\in \NN}$ via
  \begin{equation}
    \label{eq:8}
    \Bm_p \in \argmin_{\Bm} \left( \frac2{\tau} \tr\left(\left( \Bm
          - \Bm_{p-1}\right) \left( \Bm
          - \Bm_{p-1}\right)^T \right)+ E_n(\Bm\Bm^T) \right) 
  \end{equation}
  where $\tau>0$ is a parameter. Then
  \begin{equation}
    \label{eq:9}
    \Bm_{p}\Bm_p^T \to A(t) \qquad \text{as} \ \ \tau \to 0, \
    p\to\infty \ \ \text{with} \ \ p\tau \to t 
  \end{equation}
  where $A(t)$ solves~\eqref{eq:QGD} for the initial condition
  $A(0)=\Bm_0\Bm_0^T$.
\end{proposition}

\begin{proof}
  Look for a solution to the minimization
  problem in~\eqref{eq:8} of the form
  \begin{displaymath}
    \Bm = \Bm_{p-1} + \tau \tilde \Bm
  \end{displaymath}
  To leading order in $\tau$, the objective function in~\eqref{eq:8}
  becomes
  \begin{displaymath}
    \begin{aligned}
      &\frac2{\tau} \tr\left(\left( \Bm - \Bm_{p-1}\right)
        \left( \Bm - \Bm_{p-1}\right)^T\right)+
      E_n(\Bm\Bm^T) \\
      & = 2\tau \tr (\tilde \Bm\tilde \Bm^T) + \tau
      \tr\left( \left(\Bm_{p-1} \tilde\Bm^T+\tilde \Bm\Bm^T_{p-1}\right) \nabla
        E_n(\Bm_{p-1}\Bm^T_{p-1})\right) + O(\tau^2)\\
      & =  \tau \tr \left( \tilde \Bm\left(\tilde \Bm^T+ \Bm^T_{p-1}\nabla
          E_n(\Bm_{p-1}\Bm^T_{p-1})\right)\right) \\
      & + \tau \tr \left( \left(\tilde \Bm+ \nabla
          E_n(\Bm_{p-1}\Bm^T_{p-1}) \Bm_{p-1}\right)\tilde \Bm^T\right) + O(\tau^2)
    \end{aligned}
  \end{displaymath}
  which we can set to zero by choosing $\tilde \Bm = \tilde \Bm_p$ with
  \begin{displaymath}
    \tilde \Bm_p = - \nabla
        E_n(\Bm_{p-1}\Bm^T_{p-1}) \Bm_{p-1}+ O(\tau)
    \end{displaymath}
    In terms of the minimizer $\Bm_p$ of the orginal problem this equation
    can be written as 
    \begin{displaymath}
    \tau^{-1}\left(\Bm_p - \Bm_{p-1} \right) = - \nabla
        E_n(\Bm_{p-1}\Bm^T_{p-1}) \Bm_{p-1} + O(\tau)
      \end{displaymath}
      Letting $\tau\to0$ and $p\to\infty$ with $p\tau\to t$, we deduce
      that $\Bm_p \to \Bm(t)$ solution to
      \begin{equation}
        \label{eq:Bm}
        \dot \Bm(t) = - \nabla
        E_n(\Bm(t)\Bm^T(t)) \Bm(t)
      \end{equation}
      Setting $A(t) = \Bm(t) \Bm^T(t)$ we have
      \begin{displaymath}
        \begin{aligned}
          \dot A(t) &= \dot \Bm(t) \Bm^T(t) + \Bm(t) \dot \Bm^T(t)\\
          & = - \nabla
        E_n(\Bm(t) \Bm^T(t)) \Bm(t) \Bm^T(t) - \Bm(t) \Bm^T(t)  \nabla
        E_n(\Bm(t) \Bm^T(t)) \\
        & = - \nabla
        E_n(A(t)) A(t) - A(t)  \nabla
        E_n(A(t))
        \end{aligned}
      \end{displaymath}
     which is~\eqref{eq:QGD}.
\end{proof}


\subsection{Proof of Theorem~\ref{th:strict_convexity_loss}}\label{si:geometrical_consideration_prop_strict}

Let $A_{n,d}$ be a symmetric, positive semidefinite minimizer of the empirical loss and  consider $A_{n,d}-A^*$. Since this matrix is symmetric, there exists an orthonormal basis in $\RR^{d}$  made of its eigenvectors, $\{\vb_i\}_{i=1}^{d}$. Since $A_{n,d}$ is positive semidefinite by assumption and $A^*=\wb^* (\wb^*)^T$ is rank one, $d-1$ eigenvalues of $A_{n,d}-A^*$ are nonnegative, and only one can be positive, negative, or zero. Let us order the eigenvectors  $\vb_i$ such that their associate eigenvalues are $\lambda_i\ge0$ for $i=1,\ldots,d$ and $\lambda_d \in \RR$. Given the data vector $\{\xb_k\}_{k=1}^n$, to be a minimizer of the empirical loss $A_{n,d}$ must satisfy
\begin{equation}
    \label{eq:C1}
    \forall k=1,\ldots,n \quad : \quad 0= \tr  [\Xm_k (A_{n,d}-A^*)] = \< \xb_k, (A_{n,d}-A^*) \xb_k\> = \sum_{i=1}^d \lambda_i |\xb_k \cdot \vb_i|^2
\end{equation}
Let us analyze when \eqref{eq:C1} admits solutions that are not $A^*$. To this end, assume first that $\lambda_d\ge0$. Then, as soon as $n\ge d$, for each $i\in\{1,\ldots,d\}$ with probability one there is at least one $k\in\{1,\ldots,n\}$ such that $\xb_k \cdot \vb_i\not=0$. As a result, if $\lambda_d\ge 0$, as as soon as $n\ge d$, the only solution to~\eqref{eq:C1} is $\lambda_i =0$ for all $i=1,\ldots,d$, i.e. $A_{n,d}=A^*$ a.s.
    
The worst scenario case is actually when $\lambda_d<0$. In that case \eqref{eq:C1} can be written
\begin{equation}
    \label{eq:C1b}
    \forall k=1,\ldots,n \qquad : \quad \sum_{i=1}^{d-1} \lambda_i |\xb_k \cdot \vb_i|^2 = |\lambda_d| |\xb_k \cdot \vb_d|^2
\end{equation}    
This equation means that if we let $\hat \xb_k = \xb_k \sign(\xb_k \cdot \vb_d)$ (i.e. $\hat \xb_k \parallel \xb_k$ but lie in the same hemisphere as $\vb_d$), then the vectors $\hat \xb_k$ must all lie on the surface of an elliptical cone $C$ centered around $\vb_d$, with the principal axes of the ellipsoids aligned with $\vb_i$, $i=1,\ldots,d-1$; the intersection of the cone with the hyperplane $\xb\cdot \vb_d=1$ is the $d-1$ ellipsoid whose boundary satisfies the equation
\begin{equation}
    \label{eq:C1c}
    \sum_{i=1}^{d-1} \lambda_i |\xb \cdot \vb_i|^2 = |\lambda_d| 
\end{equation}
In $\RR^d$, it takes $\frac12d(d+1)$ vectors $\hat \xb_k$ to uniquely define such a elliptical cone. This means that, in the worst case scenario, we recover the threshold $n=\frac12d(d+1)$. This worst case scenario is however unlikely. To see why, assume that $n\ge d$, and consider the convex polyhedral cone spanned by $\{\hat \xb_k\}_{k=1}^n$, i.e. the region
\begin{equation}
    \label{eq:polycone}
    C_{n,d} = \left\{ \xb : \xb={\textstyle\sum_{k=1}^n }\alpha_k \hat \xb_k, \alpha_k\ge0, k=1,\ldots,n\right\}\subset \RR^d
\end{equation}
In order that~\eqref{eq:C1c} have a nontrivial solution, the extremal rays of $C_{n,d}$ (i.e. its edges of dimension 1) must coincide with the set $\{\hat \xb_k\}_{k=1}^n$, that is, all rays $\alpha_k \hat \xb_k$, $\alpha_k \ge0$ for $k=1,\ldots,n$ must lie on the boundary of $C_{n,d}$ and none can be in the interior of $C_{n,d}$; indeed these extremal rays must also be on the boundary of elliptical cone $C$. However, Theorem 3' in~\cite{cover1967geometrical} asserts that,  if the vectors in the set $\{\hat \xb_k\}_{k=1}^n$ are in general position (i.e. if the vectors in any subset of size no more than $d$ are linearly independent, which happens with probability one if $\xb_k$ are i.i.d. Gaussian), the number $N_{n,d}$ of extremal rays of $C_{n,d}$ satisfies
\begin{equation}
    \EE_\nu N_{n,d} = 2n \frac{C(n-1,d-1)}{C(n,d)}, \qquad C(n,d) = 2\sum_{k=0}^{d-1} \binom{n-1}{k}
\end{equation}
This implies that
\begin{equation}
    \label{eq:limN}
    \lim_{d\to\infty}d^{-1} \EE_\nu
    N_{\lfloor\alpha d\rfloor,d} = 
        \left\{\begin{aligned}
        \alpha& \qquad &&\text{if} \ \alpha \in[1,2]\\
        2 & &&\text{if} \ \alpha \in(2,\infty)
        \end{aligned}\right.
\end{equation}
Since $N_{n,d}\le n$ by definition, we have $d^{-1}N_{\lfloor\alpha d\rfloor,d}\le \alpha$, which from~\eqref{eq:limN} implies that $\lim_{d\to\infty}d^{-1} N_{\lfloor\alpha d\rfloor,d} = \alpha$ a.s. if $\alpha \in [1,2]$. In turns this implies that the probability that all the vectors in $\{\hat\xb_k\}_{k=1}^n$ be extremal ray of the cone $C_{n,d}$ tends to 1 as $d,n\to\infty$ with $n=\lfloor\alpha d\rfloor$ and $\alpha\in[0,2]$. This also means that the probability that \eqref{eq:C1c} has solution with $\lambda_d<0$ also tends to 1 in this limit, i.e. \eqref{eq:lim1} holds. Conversely, since $\lim_{d\to\infty}d^{-1} N_{\lfloor\alpha d\rfloor,d} = 2<\alpha$ for $\alpha>2$, the probability that $N_{n,d} \not= n $ remains positive as $d,n\to\infty$ with $n=\lfloor\alpha d\rfloor$ and $\alpha\in(2,\infty)$.  This  means that the probability that \eqref{eq:C1c} has no solution with $\lambda_d<0$ is positive in this limit, i.e. \eqref{eq:lim2} holds. \hfill $\square$

\subsection{Heuristic argument for arbitrary $m^*$ and $d$}\label{si:geometrical_consideration_heuristic}

Minimizers of the empirical loss satisfy:
\begin{equation}
  \label{eq:step1}
      \forall k=1,\ldots, n \quad : \quad \tr [\Xm_k (A-A^*)] = \<\xb_k, (A-A^*) \xb_k\> =0   
  \end{equation}
Clearly $A= A^*$ is always a solution to this set of equation. The question is: how large should $n$ be in order that $A= A^*$ be the only solution to that equation? If $A $ was an arbitrary symmetric matrix, we already know the answer: with probability one, we need $n\ge \frac12d(d+1)$. What makes the problem more complicated is that $A$ is required positive semidefinite. If we assume that $A^*$ has rank $m^*<d$, this implies that $C=A-A^*$ must be a symmetric matrix with $d-m^*$ nonnegative eigenvalues and $m^*$ eigenvalues whose sign is unconstrained, and we need to understand what this requirement  imposes on the solution to~\eqref{eq:step1}.

In the trivial case when $m^*=0$ (i.e. $A^*=0$), if we decompose $A= U \Lambda U^T$, where $U$ contains its eigenvectors and $\Lambda$ is a diagonal matrix with its eigenvectors $\lambda_i\ge0$, $i=1,\ldots, d$, \eqref{eq:step1} can be written as 
\begin{equation}
  \label{eq:step2}  
      \forall k=1,\ldots, n \quad : \quad \sum_{i=1}^d \lambda_i (\vb_i     \cdot \xb_k)^2 =0   
\end{equation}    
where $\vb_i$, $i=1,\ldots, d$ are linearly independent eigenvectors of $A$. In this case, since $\lambda_i\ge0$, with probability one we only need $n=d$ data vectors to guarantee that the only solution to this equation is $\lambda_i=0$ for all $i=1,\ldots, d$, i.e. $A=0$. Another way to think about this is to realize that the nonnegativity constraint on $A$ has removed $\frac12d(d-1)$ degrees of freedom from the original $\frac12d(d+1)$ in $A$.

If $m^*>0$, the situation is more complicated, but we can consider the projection of $A$ in the subspace not spanned by $A^*$, i.e. the $(d-m^*)\times (d-m^*)$ matrix $A^\perp$ defined as
\begin{equation}
    A^\perp = (V^*)^T A V^*
\end{equation}
where $V^*$ is the $d\times (d-m^*)$ matrix whose columns are  linearly independent eigenvectors of $A$ with zero eigenvalue. All the eigenvalues of $A^\perp$ are nonnegative, and this imposes  $\frac12 (d-m^*) (d-m^*-1)$ constraints in the subspace where $A^\perp $ lives. If we simply subtract this number to $\frac12 d(d+1)$ we obtain
\begin{equation}
    \label{eq:step3}
    n_c = \tfrac12 d(d+1)-\tfrac12 (d-m^*) (d-m^*-1) = d(m^*+1) - \tfrac12 m^*(m^*+1)
\end{equation}
which is precisely \eqref{eq:alphac}.

This argument is nonrigorous because we cannot \textit{a~priori} treat separately~\eqref{eq:step1} in the subspace spanned by $A^*$ and its orthogonal complement. Yet, our numerical results suggest that this assumption is valid, at least as $d\to\infty.$

\subsection{Proof of Theorem~\ref{th:con}}
\label{si:con}

Since \eqref{eq:w_evolution_empirical} is a standard gradient flow, $\dot \wb_i(t) =-m \partial_{\wb_i} L_n$ and the loss is a quartic polynomial in the weights, we can invoke the Stable Manifold Theorem~\cite{smale1963stable} to conclude that the stable manifolds of local minimizers of $L_n$ have codimension 0, whereas the stable manifolds of all other critical points have codimension strictly larger than 0. As a result, the weights must converge towards a local minimizer of the loss with probability one with respect to random initial data drawn for any probability distribution that is absolutely continuous with respect to the Lebesgue measure on $\RR^{md}$: this is the case under our assumption on $\{\wb_i(0)\}_{i=1}^m$.  Denoting $\wb_i^\infty=\lim_{t\to\infty} \wb_i(t)$, since $\{\wb_i^\infty\}_{i=1}^m$ is a local minimizer of $L_n$,  there exits $\delta>0$ such that 
\begin{equation}
\label{eq:locminw}
    \frac 1m \sum_{i=1}^m |\wb_i-\wb_i^\infty|^2 \le \delta \qquad \Rightarrow \qquad L_n(\wb_i) \ge L_n(\wb_i^\infty)
\end{equation}
Since the GD flow in \eqref{eq:w_evolution_empirical} for the weights implies \eqref{eq:QGD} as evolution equation for $A(t) = m^{-1}\sum_{i=1}^m \wb_i(t) \wb_i^T(t)$ by Lemma~\ref{lem:1},  it follows that $\lim_{t\to\infty} A(t) = A_\infty = m^{-1}\sum_{i=1}^m \wb_i^\infty \wb_i^\infty$.  As soon as $m\ge d$, any symmetric positive semidefinite $A$ can be constructed via a set of weights $\{\wb_i\}_{i=1}^m$, i.e. 
\begin{equation}
   \forall A=A^T \ \text{PSD}  \quad \exists \{\wb_i\}_{i=1}^m \quad : \quad A= \frac 1m \sum_{i=1}^m \wb_i\wb_i^T \quad \& \quad L_n(\wb_1,\ldots,\wb_m) = 2 E_n(A)
\end{equation}
This implies that 
$A_\infty$ must be a local minimizer of the empirical loss $E_n(A)$, otherwise for any $\epsilon>0$ there would be a $A$ such that
\begin{equation}
    \tr\big[(A-A_\infty)^2\big] \le \epsilon \quad \& \quad E_n(A) < E_n(A_\infty)
\end{equation}
Choosing $\{\wb_i\}_{i=1}^m$ such that $A= \frac 1m \sum_{i=1}^m \wb_i\wb_i^T$ would contradict~\eqref{eq:locminw}. This also implies~\eqref{eq:convemp2} since all the minimizers of the empirical loss are global minimizers by convexity, and  $E_n(A_\infty)=E_n(A^*)=0$.\hfill $\square$

\subsection{Proof of Theorem~\ref{th:convrate}}
\label{si:conrate}

We begin with:
\begin{proof}[Proof of~\eqref{eq:convpop1} in Theorem~\ref{th:convrate}]
 We can follow the same steps as in the proof of Theorem~\ref{th:con}, using \eqref{eq:w_evolution_pop} instead of \eqref{eq:w_evolution_empirical}, and noticing that this equation is also a standard gradient flow, $\dot \wb_i(t) =-m \partial_{\wb_i} L$ on the quartic loss 
\begin{equation}
    L(\wb_1,\ldots,\wb_m) = 2 E(A) \qquad \text{with} \ \ A=\frac1m \sum_{i=1}^m \wb_i \wb_i^T
\end{equation}
and $E(A)$ given in~\ref{eq:4}. The only difference is that the minimizer of $E(A)$ is now unique and given by $A^*$, which guarantees~\eqref{eq:convpop1}.
\end{proof}

This leaves us with establishing the convergence rates in~\eqref{eq:Eboundnonas} and~\eqref{eq:Eboundas}.
Their proof replies on three Lemmas that we state first.

\begin{lemma}
  \label{lem:trace}
  Let $A(t)$ be the solution to the GD flow \eqref{eq:QGDpop} and assume that $A(0)$ has full rank. Then we have
  \begin{equation}
  \label{eq:repSDE}
      A^{-1}(t) = \EE[ \zb(t) \zb^T(t)] 
  \end{equation}
  where $\zb{}(t)\in\RR^d$ solves the stochastic differential equation (SDE)
  \begin{equation}
  \label{eq:sde}
      d \zb =  (\tr(A-A^*) ) \zb dt - 2  A^* \zb dt + 2 d\Wb(t)
  \end{equation}
  Here $\Wb(t)$ is a standard $d$-dimensional Wiener process and we impose that the initial condition $\zb(0)$ be Gaussian, independent of $\Wb$, with mean zero and covariance $\EE [\zb(0) \zb^T(0)] = A^{-1}(0)$.
\end{lemma}

\begin{lemma}
  \label{lem:trace2}
  Under the conditions of Lemma~\ref{lem:trace}, we have the following identity for all $t\ge \tau\ge0$:
  \begin{equation}
  \label{eq:Ainv0}
    \begin{aligned}
      A^{-1}(t) &= 
      e^{-2A^* (t-\tau)} A^{-1}(\tau) e^{-2A^* (t-\tau)} \exp\left({\textstyle 2\int_\tau^t} \tr(A(s)-A^*) ds\right)\\
    & + 4 \int_\tau^t  e^{-4A^* (t-s)}  \exp\left({\textstyle 2\int_s^t} \tr(A(u)-A^*) du\right) ds
      \end{aligned}
  \end{equation}
\end{lemma}

\begin{lemma}
  \label{lem:trace3}
  Under the conditions of Lemma~\ref{lem:trace}, we have
  \begin{equation}
  \label{eq:Ainvlim}
    \lim_{t\to\infty} \tr(A^* A^{-1}(t) A^*) = \tr A^*
  \end{equation}
\end{lemma}
Note that, if $m^*\ge d$ and $A^*$ is invertible, since $\lim_{t\to\infty} A(t) = A^*$, we have $\lim_{t\to\infty} A^{-1}(t) = (A^*)^{-1}$ and \eqref{eq:Ainvlim} trivially  holds.
This equation also holds when $m^*<d$, i.e. when $A^*$ is rank deficient and not invertible,  if we assume that $A(0)=\text{Id}$ so that $A(t)$ remains diagonal at all times by Theorem~\ref{th:QGDpop_dynamical_evolution} since, using the notations of this theorem,  we then have
\begin{equation}
\label{eq:Ainvlim3}
    \lim_{t\to\infty} \tr(A^* A^{-1}(t) A^*) = \lim_{t\to\infty} \sum_{i=1}^{m^*} \frac{(\lambda_i^*)^2}{\lambda_i(t)} = \sum_{i=1}^{m^*} \lambda_i^* = \tr A^*
\end{equation}
because $\lim_{t\to\infty}\lambda_i(t)=\lambda_i^*>0$ if $i\le m^*$. However, the dangerous case is when $m^*<d$ and $A(t)$ is not diagonal: in that case \eqref{eq:Ainvlim}  is nontrivial.

\begin{proof}[Proof of Lemma~\ref{lem:trace}]
Since $A(0)$ has full rank, $A^{-1}(0)$ exists, and 
  since $A(t)$ solves \eqref{eq:QGDpop}, $A^{-1}(t)$ satisfies
  \begin{equation}
  \label{eq:Ainveq}
  \begin{aligned}
      \frac{d}{dt} A^{-1}(t) &= 2\left[(\tr(A-A^*) ) A^{-1} + A^{-1} (A-A^*) + 
        (A-A^*)A^{-1}\right]\\
        & = 2 (\tr(A-A^*) ) A^{-1} - 2 A^{-1} A^* -2 A^* A^{-1} + 4\, \text{Id}
    \end{aligned}
  \end{equation}
  A direct calculation with~\eqref{eq:sde} using It\^o formula shows that $\EE[\zb(t) \zb^T(t)]$ solves~\eqref{eq:Ainveq} for the same initial condition, i.e. \eqref{eq:repSDE} holds.
\end{proof}

\begin{proof}[Proof of Lemma~\ref{lem:trace2}]
 Equation~\eqref{eq:repSDE} implies that
  \begin{equation}
  \label{eq:equiv1}
      \tr(A^* A^{-1}(t) A^*) = \EE| A^* \zb(t)|^2
  \end{equation}
  Since the solution to~\eqref{eq:sde} can be expressed as
  \begin{equation}
  \label{eq:sdesol}
  \begin{aligned}
      \zb(t)  =  & \exp\left(-2A^*(t-\tau)+\textstyle{\int_\tau^t}\tr(A(s)-A^*)ds \right) \zb(\tau) \\
      & + 2  \int_\tau^t \exp\left(-2A^*(t-s)+\textstyle{\int_s^t}\tr(A(u)-A^*)du \right) d\Wb(s),
     \end{aligned}
  \end{equation}
  a direct calculation using this formula in~\eqref{eq:equiv1} together with $\EE [z(\tau) z^T(\tau)] = A^{-1}(\tau)$ and It\^o isometry establishes~\eqref{eq:Ainv0}.
  \end{proof}
  
\begin{proof}[Proof of Lemma~\ref{lem:trace3}]
We only need to consider the nontrivial case when $A^*$ is rank deficient, i.e. $m^*<d$. To begin, notice that~\eqref{eq:Ainv0} implies the following identity for all $t\ge\tau\ge0$:
\begin{equation}
  \label{eq:Ainv}
    \begin{aligned}
      \tr(A^* A^{-1}(t) A^*) &= 
      \tr \left(A^* e^{-2A^* (t-\tau)} A^{-1}(\tau) e^{-2A^* (t-\tau)} A^* \right) \exp\left({\textstyle 2\int_\tau^t} \tr(A(s)-A^*) ds\right)\\
    & + 4 \int_\tau^t \tr \left(A^* e^{-4A^* (t-s)} A^* \right) \exp\left({\textstyle 2\int_s^t} \tr(A(u)-A^*) du\right) ds
      \end{aligned}
  \end{equation}
Since $A^*$ is symmetric and positive semidefinite, its eigenvalues are nonnegative and there exists an orthonormal basis made of its eigenvectors. Denote this basis by $\{\vb^*_i\}_{i=1}^d$ and let us order it in way that the corresponding eigenvalues are $\lambda^*_i>0$ for $i=1,\ldots,m^*$, and $\lambda^*_i=0$ for $i=m^*+1,\ldots,d$. Then \eqref{eq:Ainv} can be written as
\begin{equation}
  \label{eq:Ainv2}
    \begin{aligned}
      \tr(A^* A^{-1}(t) A^*) &= 
      \sum_{i=1}^{m^*} (\lambda^*_i)^2(\vb_i^*)^TA^{-1}(\tau)\vb^*_i \exp\left( -4\lambda_i^*(t-\tau) + 2{\textstyle\int_\tau^t} \tr(A(s)-A^*) ds\right)\\
    & + 4 \sum_{i=1}^{m^*}(\lambda_i^*)^2\int_\tau^t \exp\left(-4\lambda^*_i (t-s)+ 2\textstyle{\int_s^t} \tr(A(u)-A^*) du\right) ds
      \end{aligned}
  \end{equation}
  Since $|\tr(A(t)-A^*)|$ is bounded for all $t\ge0$, evaluating this expression at $\tau=0$ shows that $\tr(A^* A^{-1}(t) A^*)$ is also bounded i.e. we only need to consider what happens as $t\to\infty$. We  have
  \begin{equation}
  \label{eq:limrhs0}
  \quad \forall t\ge \tau \quad : \quad 
  \left| \frac{\int_\tau^t  \tr(A(u)-A^*) du}{2(t-\tau)}\right|\le C(\tau) := \frac12\max_{u\in[\tau,\infty)} |\tr(A(u)-A^*)|<\infty
  \end{equation}
  with $C(\tau)$ decaying to zero as $\tau\to\infty$ since $\lim_{t\to\infty} A(t) = A^*$ by~\eqref{eq:convpop1}. 
  This implies that the first term at the right hand side of~\eqref{eq:Ainv2} can be bounded as
  \begin{equation}
  \label{eq:limrhs11}
  \begin{aligned}
       &\sum_{i=1}^{m^*} (\lambda^*_i)^2(\vb_i^*)^TA^{-1}(\tau)\vb^*_i\exp\left( -4\lambda_i^*(t-\tau) + 2{\textstyle\int_\tau^t} \tr(A(s)-A^*) ds\right) \\
       &\le \sum_{i=1}^{m^*} (\lambda^*_i)^2(\vb_i^*)^TA^{-1}(\tau)\vb^*_i \exp\left( -4\lambda_i^*(t-\tau) [1- C(\tau)/\lambda_i^*]\right).
  \end{aligned}
  \end{equation}
  Since $\lim_{\tau\to\infty} C(\tau)=0$, there exists $\tau_c\ge0$ such that 
  $C(\tau)<\min_{i=1,\ldots,m^*} \lambda_i^*$ for all $\tau\ge\tau_c$, and hence $1- C(\tau)/\lambda_i^*>0$ for all $\tau\ge\tau_c$ and for all $i=1,\ldots,m^*$. Therefore we can let $t\to\infty$ at any fixed  $\tau\ge \tau_c$ in \eqref{eq:limrhs11}  to conclude that the limit of the first term at the right hand side of \eqref{eq:Ainv2} is zero, i.e.
  \begin{equation}
  \label{eq:limrhs1}
  \begin{aligned}
       &\lim_{t\to\infty} \sum_{i=1}^{m^*} (\lambda^*_i)^2(\vb_i^*)^TA^{-1}(\tau)\vb^*_i\exp\left( -4\lambda_i^*(t-\tau) + 2{\textstyle\int_{\tau}^t} \tr(A(s)-A^*) ds\right) =0 \qquad (\tau\ge\tau_c).
  \end{aligned}
  \end{equation}
  Similarly, to deal with the second term at the right hand side of \eqref{eq:Ainv2}, we can use
  \begin{equation}
  \label{eq:limrhs2}
  \quad \forall t\ge s\ge \tau \quad : \quad 
  \left| \frac{\int_s^t  \tr(A(u)-A^*) du}{2(t-s)}\right|\le C(\tau) 
  \end{equation}
  with the same $C(\tau)$ as in~\eqref{eq:limrhs0}.
  As a result, by taking again $\tau\ge\tau_c$,
  we have
  \begin{equation}
  \begin{aligned}
      & \lim_{t\to\infty} 4\sum_{i=1}^{m^*}(\lambda_i^*)^2\int_\tau^t \exp\left(-4\lambda^*_i (t-s)+ 2\textstyle{\int_s^t} \tr(A(u)-A^*) du\right) ds \\
      & \le 4\sum_{i=1}^{m^*} (\lambda_i^*)^2 \lim_{t\to\infty} \int_\tau^t \exp\left(-4\lambda^*_i (t-s) [1-C(\tau)/\lambda_i^*]\right)ds\\
      & =  \sum_{i=1}^{m^*}\lambda_i^* [ 1-C(\tau)/\lambda_i^*]^{-1} \qquad (\tau\ge\tau_c)
      \end{aligned} 
  \end{equation}
  Therefore we have established that
  \begin{equation}
  \label{eq:Ainv3}
      \lim_{t\to\infty} \tr(A^* A^{-1}(t) A^*) \le \sum_{i=1}^{m^*}\lambda_i^* [ 1-C(\tau)/\lambda_i^*]^{-1}\qquad  (\tau\ge\tau_c)
  \end{equation}
  Since this equation is valid for any $\tau\ge\tau_c$ and $\lim_{\tau\to\infty} C(\tau)=0$, we can now let $\tau\to\infty$ on the right hand side of~\eqref{eq:Ainv3} to deduce
  \begin{equation}
  \label{eq:upperb}
      \lim_{t\to\infty} \tr(A^* A^{-1}(t) A^*) \le \sum_{i=1}^{m^*}\lambda_i^* =\tr A^*
  \end{equation}
  To get the matching lower bound, use $\tr\left((A(t)-A^*)
      A^{-1}(t)(A(t)-A^*)\right)\ge 0$ to deduce
  \begin{equation}
  \tr(A^* A^{-1}(t) A^*) \ge 2\tr A^* - \tr A(t)
  \end{equation}
  and take the limit as $t\to\infty$ using $\lim_{t\to\infty}\tr A(t) = A^*$ to obtain 
  \begin{equation}
  \label{eq:lowerb}
      \lim_{t\to\infty} \tr(A^* A^{-1}(t) A^*) \ge \tr A^*
  \end{equation}
  Taken together \eqref{eq:upperb} and \eqref{eq:lowerb} imply \eqref{eq:Ainvlim}.
\end{proof}

We can now use these results to proceed with the rest of the proof of Theorem~\ref{th:convrate}:

\begin{proof}[Proof of~\eqref{eq:Eboundnonas} and~\eqref{eq:Eboundas} in Theorem~\ref{th:convrate}]
 The multiplicative inverse of $E(A(t))$ satisfies
  \begin{equation}
  \label{eq:Erecip}
    \frac{d}{dt} E^{-1}(A(t)) = 2 E^{-2}(A(t)) \tr[\nabla E
    (A(t)) A(t)\
    \nabla E (A(t))].
  \end{equation}
  By convexity of $E(A)$ we have
  \begin{equation}
  E(A(t)) \le \tr[(A(t)-A^*)\nabla E(A(t))] =
  \tr[(A(t)-A^*) A^{-1/2}(t) A^{1/2}(t) \nabla E(A(t))] 
\end{equation}
where we used the positivity of $A(t)$ as well as $E(A^*)=0$. Therefore
using Cauchy-Schwarz inequality we obtain
\begin{equation}
  E^2(A(t))  \le \tr[\nabla E
  (A(t)) A(t)\
  \nabla  E (A(t))] \tr\left[(A^*(t)-A^*)A^{-1}(t)
    (A(t)-A^*)\right].
\end{equation}
Using this inequality in~\eqref{eq:Erecip}  we deduce
\begin{equation*}
  \frac{d}{dt}E^{-1}(A(t))\ge 2 \left[\tr\left((A(t)-A^*)A^{-1}(t)
    (A(t)-A^*)\right)\right]^{-1}
\end{equation*}
Integrating and reorganizing gives
\begin{equation}
\label{eq:Ebound}
  E(A(t)) \le \frac{E(A(0))}{1+2
    E(A(0))\int_0^t \left[\tr\left((A(s)-A^*)
      A^{-1}(s)(A(s)-A^*)\right)\right]^{-1} ds}
\end{equation}
To proceed let us analyze the behavior of the integral in the denominator. Start by noticing that
\begin{equation}
\label{eq:limexists}
\begin{aligned}
    &\lim_{t\to\infty} \tr\left((A(t)-A^*)
      A^{-1}(t)(A(t)-A^*)\right) \\
      & = \lim_{t\to\infty}\tr A(t) - 2 \tr A^* + \lim_{t\to\infty}\tr (A^*A^{-1}(t) A^*)=0
      \end{aligned}
\end{equation}
where we used $\lim_{t\to\infty}\tr A(t)=\tr A^*$ as well as \eqref{eq:Ainvlim} in Lemma~\ref{lem:trace2}. \eqref{eq:limexists} guarantees that $\tr\left[(A(t)-A^*)
 A^{-1}(t)(A(t)-A^*)\right] $ is bounded for all time, i.e.
 \begin{equation}
    \forall t\ge0 \quad : \quad \left[\tr\left((A(t)-A^*)
 A^{-1}(t)(A(t)-A^*)\right)\right]^{-1} \ge C >0
\end{equation}
with 
\begin{equation}
    C = \left[\max_{t\in[0,\infty)} \tr\left[(A(t)-A^*)
 A^{-1}(t)(A(t)-A^*)\right]\right]^{-1} 
\end{equation}
As a result
\begin{equation}
 \forall t\ge0 \quad : \quad \int_0^t \tr\left[(A(s)-A^*)
 A^{-1}(s)(A(s)-A^*)\right]^{-1} ds \ge Ct
\end{equation}
which from~\eqref{eq:Ebound} implies the nonasymptotic bound in~\eqref{eq:Eboundnonas}. To establish the asymptotic bound in~\eqref{eq:Eboundas}, note that~\eqref{eq:limexists} implies that 
\begin{equation}
 \int_0^t \left[\tr\left((A(s)-A^*)
 A^{-1}(s)(A(s)-A^*)\right)\right]^{-1} ds \ \  \text{grows faster than $t$ as $t\to\infty$}
\end{equation}
Using this result in~\eqref{eq:Ebound} implies that $E(A(t))$ decays faster than $1/t$ as $t\to\infty$, i.e. $E(A(t))=o(1/t)$ and \eqref{eq:Eboundas} holds.
\end{proof}

\subsection{Proof of Theorem~\ref{th:QGDpop_dynamical_evolution}}
\label{si:popdiag}

Since $A^*$ is symmetric and positive semidefinite, its eigenvalues are nonnegative and there exists an orthonormal basis made of its eigenvectors. Denote this basis by $\{\vb^*_i\}_{i=1}^d$ and let us order it in way that the corresponding eigenvalues are $\lambda^*_i>0$ for $i=1,\ldots,m^*$, and $\lambda^*_i=0$ for $i=m^*+1,\ldots,d$. Denote by $U^*$ the orthogonal matrix whose columns are the eigenvectors of $A^*$, so that $A^* = U^* \Lambda^* (U^*)^T$ with $\Lambda^* = \text{diag}(\lambda^*_1, \ldots, \lambda^*_d)$. Let $\Lambda(t) = (U^*)^T A(t) U^*$. Since $A(0) = \text{Id}$ by assumption, $\Lambda(0) = \text{Id}$ and from~\eqref{eq:QGDpop} this matrix evolves according to
\begin{equation}
    \label{eq:Lambda}
    \begin{aligned}
    \dot\Lambda & = 2 (\tr(A^*-A)) (U^*)^T A U^* + 2(U^*)^T (A^*-A) A  U^*+ 2(U^*)^T A (A^*-A)   U^*\\
    & = 2 (\tr(\Lambda^*-\Lambda )) \Lambda + 2\Lambda (\Lambda^*-\Lambda)+ 2 (\Lambda^*-\Lambda) \Lambda.
    \end{aligned}
\end{equation}
This equation shows that $\Lambda(t)$ remains diagonal for all times, $\Lambda(t) = \text{diag}(\lambda_1(t), \ldots, \lambda_d(t))$. Written componentwise \eqref{eq:Lambda} is \eqref{eq:Lambdacomp}. \hfill$\square$

\section{Additional Results}\label{si:numerical}

\subsection{Supporting numerical results to Fig.~\ref{fig:dyn_phases}}

\begin{figure}[H]
	\centering
	\includegraphics[width=\linewidth]{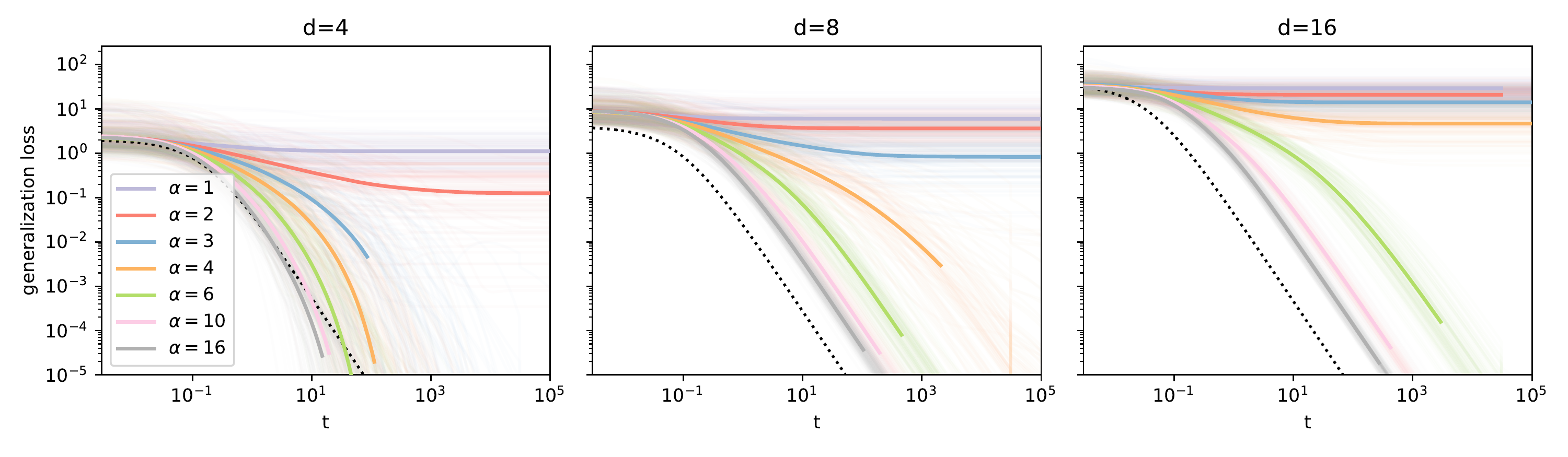}
	\caption{Population loss for $d=4,8,16$ and $m^*=4$ and several values of $\alpha=n/d$. The line shown with full color are average of the logarithm of 100 simulations (300 for $d=4$) and the individual instances are shown in transparency.}
	\label{fig:loss_g_comparison_d_nT4}
\end{figure}

In Fig.~\ref{fig:loss_g_comparison_d_nT4} shows the average performance of GD with $n=\alpha d$ datapoints and a teacher with $m^*$ and Gaussian hidden units. The figure is intended to show a vertical cut in the dynamical phases Fig.~\ref{fig:dyn_phases}. Moving up in $d$ at $m^*$ fixed we observe that on average  the simulations  converge when $\alpha>\alpha_c$ and they do not when $\alpha>\alpha_c$, i.e. there is an abrupt change of behavior when we cross the transition. Another interesting aspect of the figure is that the first panel has $m^*\ge d$ which leads to and exponential (rather than quadratic) convergence rate in the loss, consistent to our analysis. The dotted line is a reference line that represents the $1/t^2$  decay of the loss.

\subsection{Supporting numerical results to Theorem~\ref{th:strict_convexity_loss}}

\begin{figure}[H]
	\centering
	\includegraphics[width=.49\linewidth]{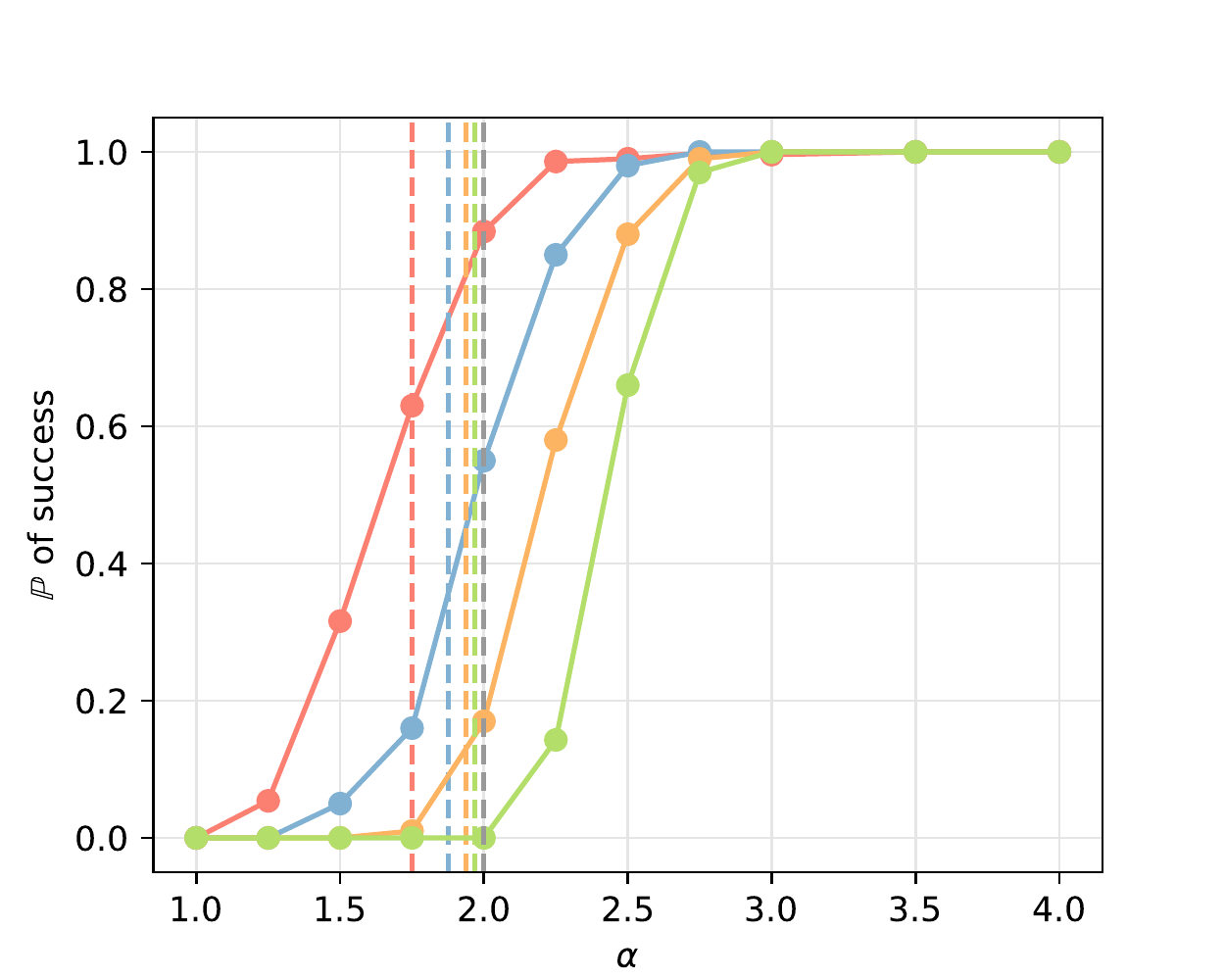}
	\includegraphics[width=.49\linewidth]{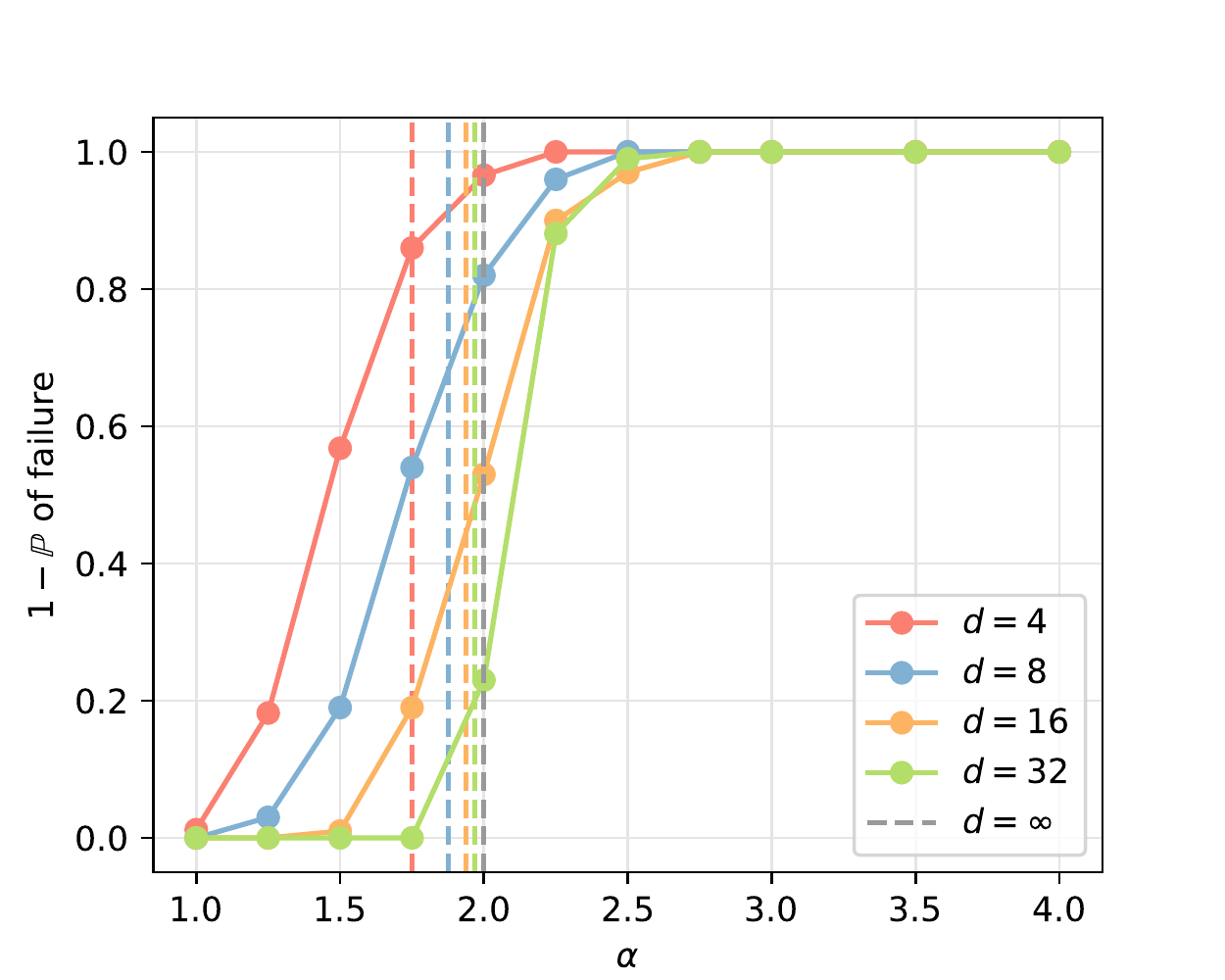}
	\caption{Left panel: fraction of simulations that went below $10^{-5}$ for $d=4,8,16,32$. Right panel: complement of the fraction of simulations that have a ratio between final generalization loss and training loss that is larger then $10^9 d$.}
	\label{fig:success_failure}
\end{figure}

In Fig.~\ref{fig:success_failure} we present a numerical verification of Theorem~\ref{th:strict_convexity_loss}. According to the theorem, as $d\to\infty$ with $m^*=1$ (so that $\alpha_c=2$) the probability of finding the teacher should converge to zero for $\alpha<2$ and to  positive values for $\alpha >2$. The left panel on the figure shows the fraction of 100 simulations that achieved at least $10^{-5}$ generalization loss after $2\log_2d\times10^7$ iterations with learning rate $0.003$. The right panel shows the number of simulations for which the ratio between training and generalization losses is larger than $10^{-9} d^{-1}$. This second panel is meant to capture  the simulations for which we expect convergence eventually, but the number of iterations was not enough to achieve it. In particular, we observed that when generalization fails, meaning that the training loss goes to zero and the generalization loss stay at a high value, the convergence rate of the training loss is exponential, contrarily to simulation where the generalization loss eventually goes to zero that have a $O(1/t^2)$ convergence rate. Using simulations with $10^8$ iterations is sufficient to detect the difference between the two cases and therefore this gives us a good criterion to distinguish between successful and unsuccessful simulations.

\begin{figure}[H]
	\centering
	\includegraphics[width=\linewidth]{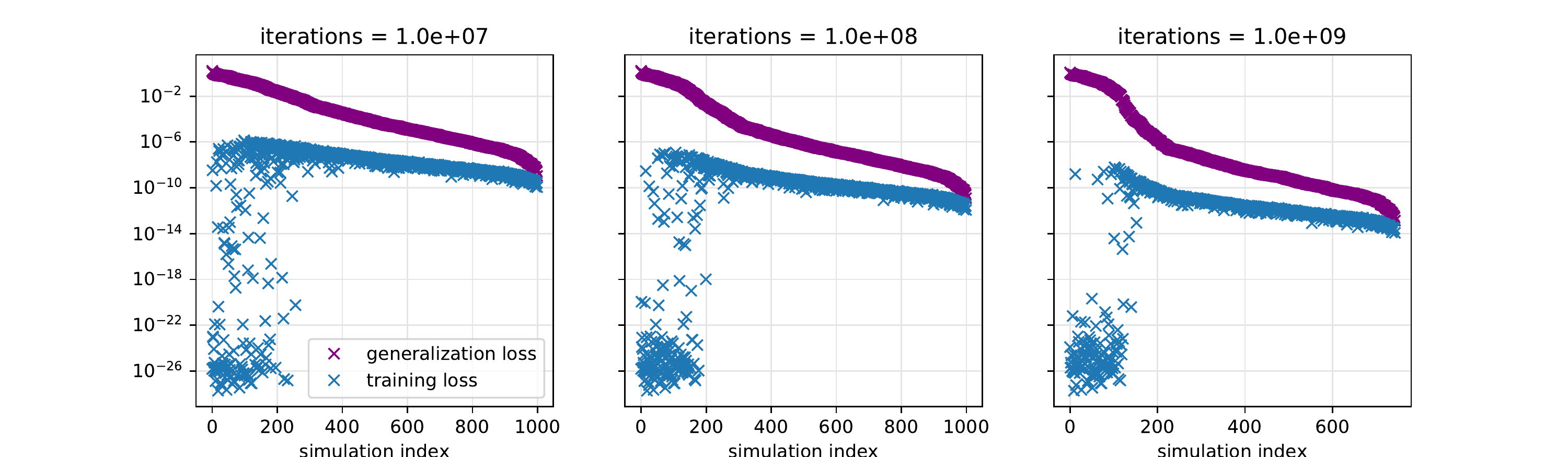}
	\caption{Final value of the training and generalization loss of several simulations with input $d=4$ and $n=7$ samples in the dataset. From left to right the maximum number of steps in the simulation increases by a factor 10.}
	\label{fig:final_values}
\end{figure}

To provide more evidence  of this reasoning, in Fig.~\ref{fig:final_values}  we show training and generalization loss of $1000$ simulations for $d=4$, $n=2d-1$ and $m^*=1$. We order the simulations according to the loss and show in the three panels three snapshots for different number of iterations. From left to right the number of iterations increases by a factor 10 in each panel. As can be seen, the ratio between generalization loss and training loss at the end of the training is a valid measure of success. 


\subsection{Extrapolation procedure}\label{sec:simulations_extrapolation}

\begin{figure}[H]
	\centering
	\includegraphics[width=.5\linewidth]{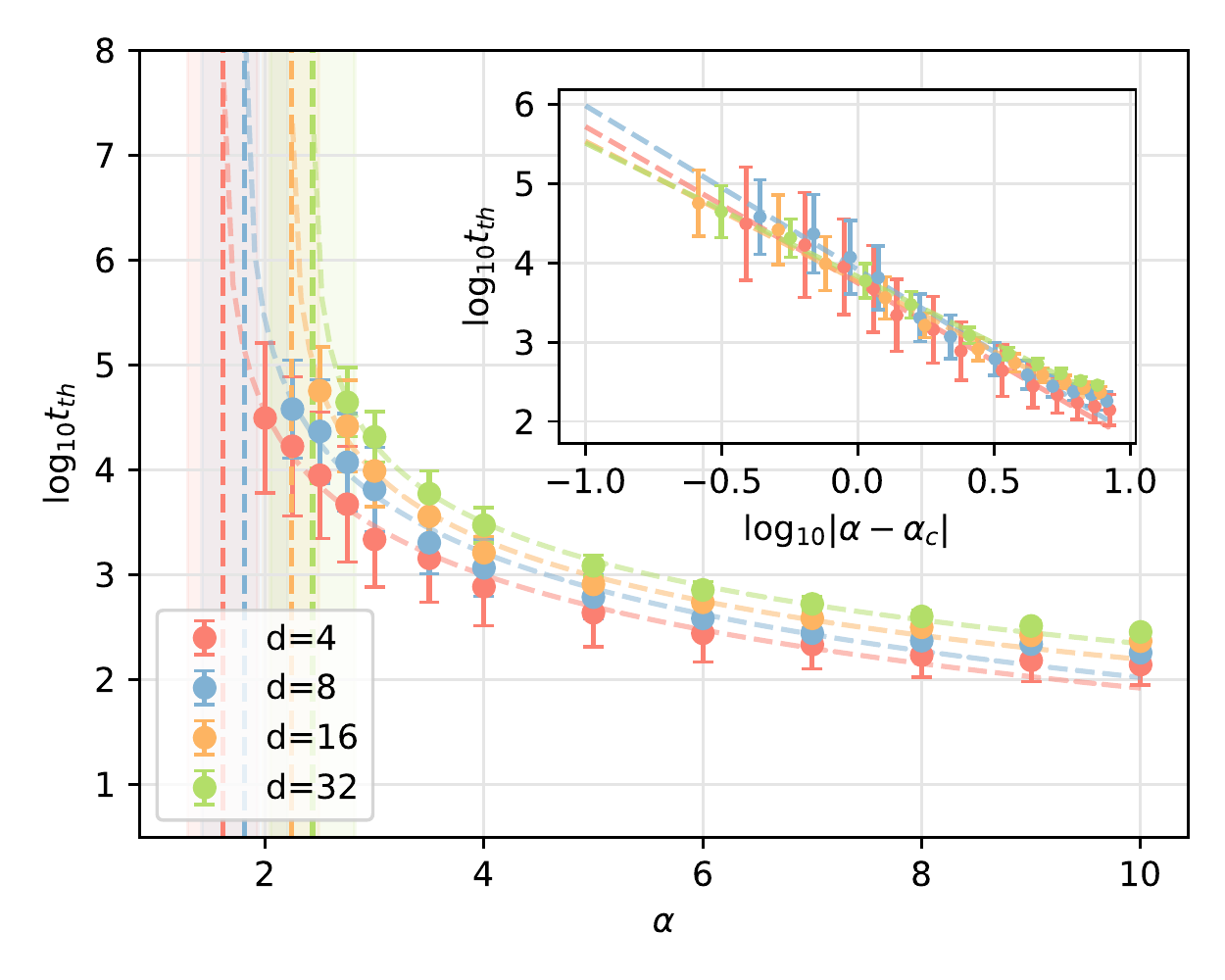}
	\caption{Extrapolation of the sample complexity threshold $\alpha_c=2$ for $m^*=1$ assuming a power-law increase of the time to converge to a $10^{-5}$ value of the loss when approaching this threshold.  In the inset we show that the points lie on a line in log-log scale.}
	\label{fig:extrapolation_alpha_c}
\end{figure}

We estimate the critical value of $\alpha$ numerically by fixing a threshold in the population loss, $10^{-5}$, and simulate the problem for a large set of $\alpha$. Starting from the largest value in the set, as $\alpha$ approaches the critical value the time needed to pass the threshold increase as a power-law $\sim|\alpha-\alpha_c|^{-\theta}$ \cite{cugliandolo1993analytical}. In Fig.~\ref{fig:extrapolation_alpha_c} we fit the relaxation times to cross a threshold in the population loss of $10^{-5}$ for $d=4,8,16,32$ and $m^*=1$. The extrapolated thresholds $\alpha_c$ and their 95\% confidence intervals are: for $d=4$, $\alpha_c=1.6$ $(1.3, 1.9)$; for $d=8$, $\alpha_c=1.8$ $(1.4, 2.2)$; for $d=16$, $\alpha_c=2.2$ $(2.0, 2.5)$; and for $d=32$, $\alpha_c=2.4$ $(2.0, 2.8)$. Close to the threshold $\alpha_c=2-1/d$, namely $1.8$, $1.9$, $1.9$, and $2.0$, as expected. The larger the input dimension, the larger the time to pass the threshold is, and as result the smallest accessible value of $\alpha$  also increases. This causes a decrease in accuracy on the threshold value, measured by the larger confidence intervals obtained assuming a t-student distribution. The same procedure has been applied for other values of $m^*$ to obtain the points shown in Fig.~\ref{fig:dyn_phases}.

\subsection{GD in the populations loss with orthogonal teacher}

\begin{figure}[H]
	\centering
	\includegraphics[width=.49\linewidth]{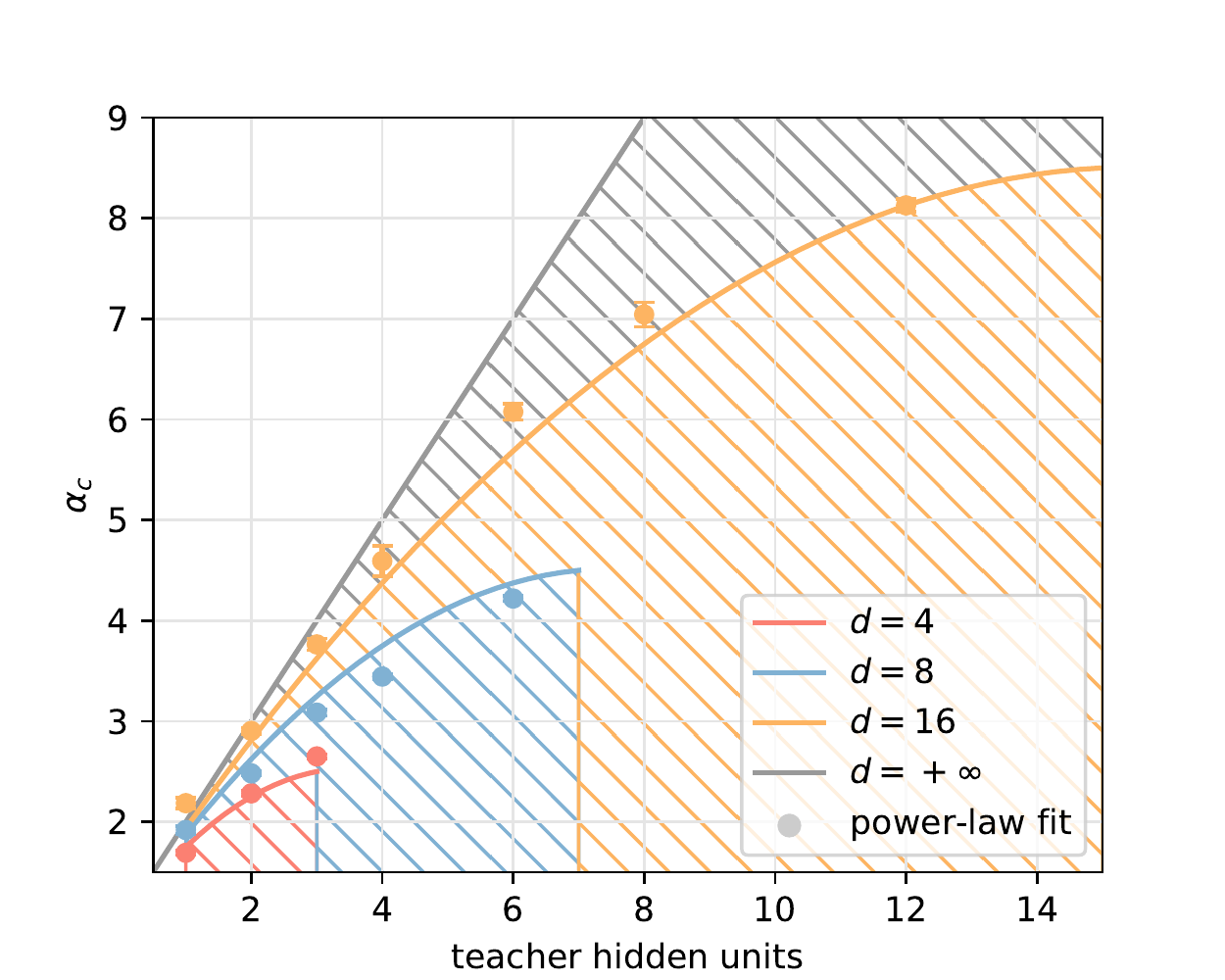}
	\caption{Same as Fig.~\ref{fig:dyn_phases} in the main text but for a  teacher with orthonormal hidden nodes. In that case,  as soon as $m^*$ becomes equal to or larger than $d$, $A^*=\text{Id}$, and therefore the student equal the teacher at initialization since $A(0) = \text{Id}$. }
	\label{fig:dyn_phases_plus}
\end{figure}

A simple special case of \eqref{eq:Lambdacomp} in Theorem~\ref{th:QGDpop_dynamical_evolution} is when the teacher has orthogonal hidden weights, so that $\lambda^*_j=1$ for  $1\le j\le m^*$ and $\lambda^*_j=0$ for every $m^*<j<d$. (Note that the problem becomes trivial in that case when $m^*=d$ since $A(t=0)=\text{Id}=A^*$.)  In that case the first $m^*$  \textit{informative eigenvalues} are the same, $\lambda_i = \lambda $ for $1\le 1\le d-m^*$, and~\eqref{eq:eigenvalue_evolution_asymptotic_leada}-\eqref{eq:eigenvalue_evolution_asymptotic_meanaa} reduce to
\begin{align}
    \label{eq:LV_eqs_lambda}
    & \frac{d}{dt}\log\lambda = (2m^*+4)(1-\lambda) - 2(d-m^*)\epsilon,
    \\
    \label{eq:LV_eqs_epsilon}
    & \frac{d}{dt}\log\epsilon = 2m^*(1-\lambda) - 2(2+d-m^*)\epsilon.
\end{align}
Those equations are an instance of \textit{Lotka-Volterra equations} that have a long history for modeling competing species  in ecology~\cite{lotka2002contribution,volterra1927variazioni}---here,  the informative $\lambda$ and noninformative $\epsilon$  eigenvalues play the role of these species. \eqref{eq:LV_eqs_lambda}-\eqref{eq:LV_eqs_epsilon} have three fixed points in the $(\lambda,\epsilon)$ space: the unstable solutions $(0,0)$ and $(0,m^*/(2+d-m^*))$, and the stable solution $(1,0)$. The phase portrait of these equation is shown in Fig.~\ref{fig:pp}.

\begin{figure}
	\centering
	\includegraphics[scale=.3]{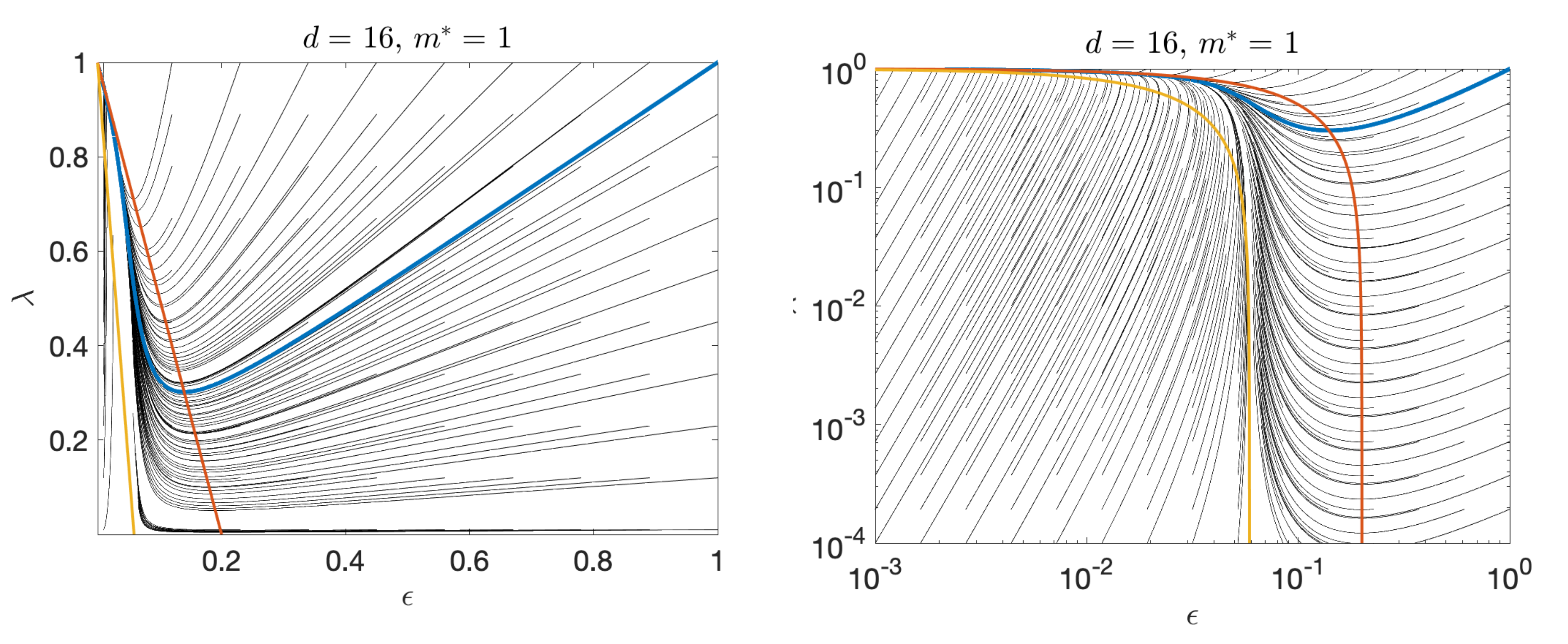}
	\caption{Phase portrait of the the Lotka-Volterra system in \eqref{eq:LV_eqs_lambda}-\eqref{eq:LV_eqs_epsilon} both in linear (left panel) and log (right panel) scales, for $d=16$ and $m^*=1$. The $\lambda$- and $\epsilon$-nullclines are shown in red and orange, respectively. The flow map is show in black. The actual solution starting from $(\lambda(0),\epsilon(0))=(1,1)$ is shown in blue. 
	}
	\label{fig:pp}
\end{figure}

Let us analyze \eqref{eq:LV_eqs_lambda}-\eqref{eq:LV_eqs_epsilon}  when $d-m^* \gg1$. In that case the dynamics of $\lambda$ and $\epsilon$ has two regimes: Initially  the second term at the right hand side of these equation is the dominant term; since this term is negative, it means that both $\lambda$ and $\epsilon$ decrease from their initial values $(\lambda(0),\epsilon(0)) = (1,1)$. In the second regime, $\epsilon$ becomes small enough that the right hand side of \eqref{eq:LV_eqs_lambda} becomes positive allowing $\lambda$ to bounce back up and grow towards its asymptotic value $\lim_{t\to\infty} \lambda(t) = \lambda^* = 1$ whereas $\epsilon$ continues to decreases so that  $\lim_{t\to\infty} \epsilon(t) = 0$ converges to zero with a linear convergence rate. If we neglect the first term at the right hand side of \eqref{eq:LV_eqs_epsilon}, this equation can be solved exactly:
\begin{equation}
    \label{eq:LV_eqs_epsilon_solON}
    \epsilon(t) \approx \frac1{1+2(2+d-m^*)t} 
\end{equation}
It turns out that this approximation is accurate in both regimes, because the first term at the right hand side of \eqref{eq:LV_eqs_epsilon} is always sub-dominant.  In the first regime, \eqref{eq:LV_eqs_lambda}-\eqref{eq:LV_eqs_epsilon} implies that $\lambda(t)\approx\epsilon(t)$, and this goes on until the right hand side of \eqref{eq:LV_eqs_lambda} changes sign, indicating the start of the second regime.
This occurs at time 
\begin{equation}
    t_0\approx\frac{d-m^*}{2(2+m^*)(2+d-m^*)}=O(1),
\end{equation} 
similarly to the random Gaussian case discussed in the main text. Observe that at that time, we have $\lambda(t_0)=\epsilon(t_0)=(m^*+2)/(d+2)$, and passed that time $\lambda(t)$ starts to increase again, while $\epsilon(t)$ keeps decreasing. Therefore in this second regime we can neglect the last term at the right hand  of \eqref{eq:LV_eqs_lambda} and solve this equation with the initial condition $\lambda(t_0)=\lambda_0=(m^*+2)/(d+2)$. This gives the logistic growth
\begin{equation}
    \label{eq:LV_eqs_lambda_solON}
    \lambda(t) \approx \frac{\lambda_0 e^{2(m^*+2)(t-t_0)}}{\lambda_0 (e^{2(m^*+2)(t-t_0)}-1) + 1}.
\end{equation}
From this equation, the time for $\lambda$ to reach its target $\lambda^*=1$ is approximately
\begin{equation}
    t_J\approx\frac1{m^*+2}\log\frac{d+2}{2(m^*+2)}.
\end{equation} 
These approximations are remarkably accurate as we can observe in Fig.~\ref{fig:approximate_eig}, where we evaluate numerically the dynamics on the population loss~\eqref{eq:QGDpop} and compare the result with the approximation for $d=512$, $m^*=1$ and $n=2048$. The left panel shows the evolution of the eigenvalues and the right one the generalization loss. The dotted line on the left is~\eqref{eq:LV_eqs_epsilon_solON} and on the right is~\eqref{eq:loss_alpha2} shown in the main text.

\begin{figure}[H]
	\centering
	\includegraphics[scale=.5]{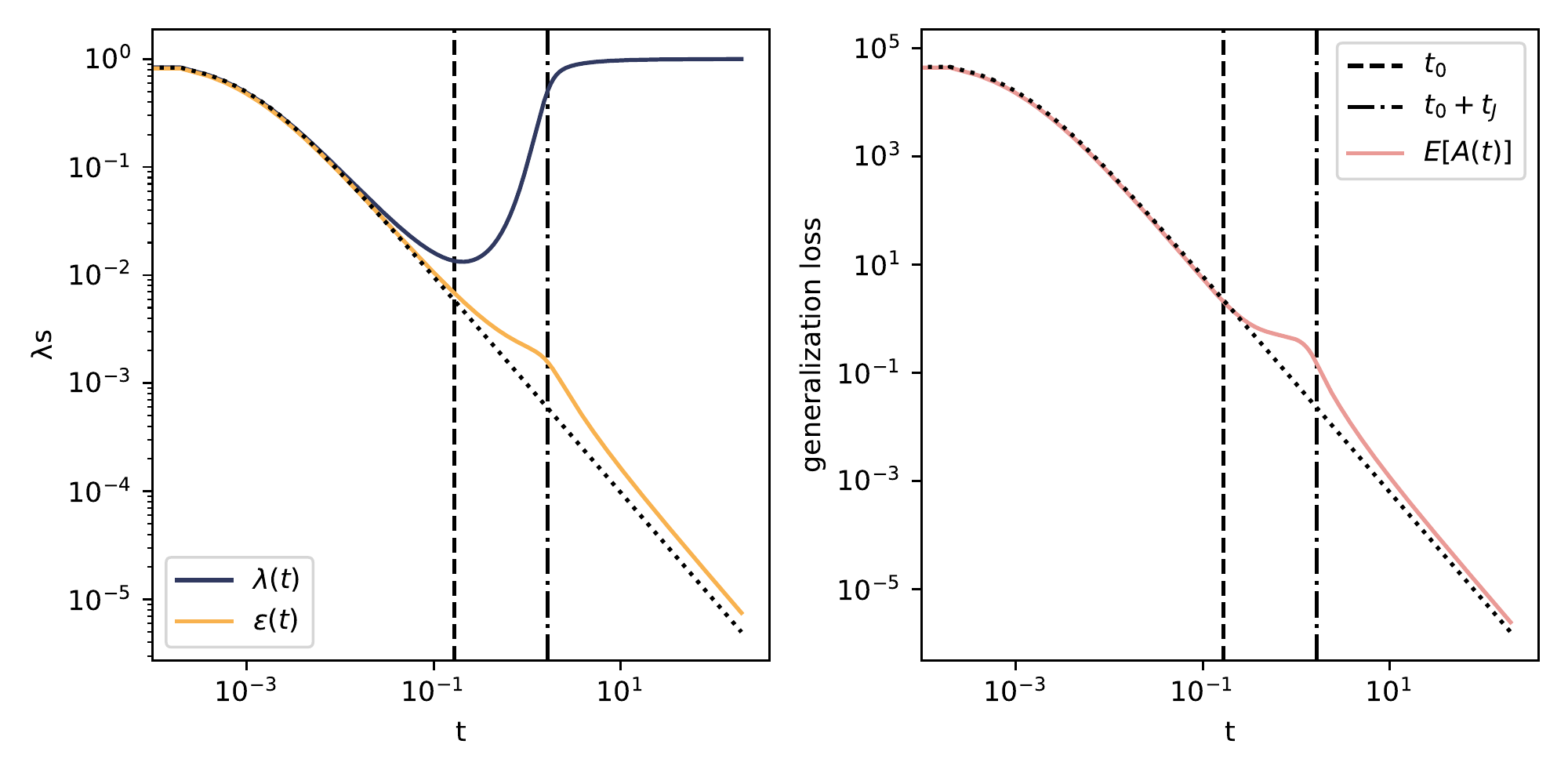}
	\caption{Evolution of the eigenvalues in the population loss (left) and generalization loss (right). Left panel: the solutions to \eqref{eq:LV_eqs_lambda}-\eqref{eq:LV_eqs_epsilon} and the approximate solution~\eqref{eq:LV_eqs_lambda_solON} (dotted line). Right panel: exact loss from~\eqref{eq:lpop_eigen} compared to its approximation in~\eqref{eq:loss_alpha_infty} (dotted line) valid for small and large times. The vertical lines show the two times $t_0$ and $t_0+t_J$. 
	}
	\label{fig:approximate_eig}
\end{figure}

\end{document}